\pdfoutput=1
\documentclass{article}

\usepackage[nonatbib,final]{neurips_2020}
\usepackage{floatrow}
\newfloatcommand{capbtabbox}{table}[][\FBwidth]
\usepackage{blindtext}
\usepackage[utf8]{inputenc} 
\usepackage[T1]{fontenc}    
\usepackage{hyperref}       
\usepackage{url}            
\usepackage{booktabs}       
\usepackage{amsfonts}       
\usepackage{nicefrac}       
\usepackage{microtype}      

\usepackage{soul}
\usepackage{url}
\usepackage{epstopdf}
\usepackage{float}
\usepackage[utf8]{inputenc}
\usepackage{graphicx}
\usepackage{amsmath}
\usepackage{amssymb}
\usepackage{booktabs}
\usepackage{algorithm}
\usepackage[noend]{algorithmic}
\usepackage{comment}
\usepackage{booktabs} 
\usepackage{graphicx,array}
\usepackage{color,soul,xspace}

\usepackage{comment}
\usepackage{epsfig}
\usepackage{subfig}
\usepackage{mathrsfs,mdwlist,enumitem}

\newcommand{\expectation}{\mathbb{E}}
\newcommand{\flow}{\Gamma}

\setlist{nolistsep,leftmargin=*}
\def\gc{\textsc{Gcomb}\xspace}
\def\gt{\textsc{Gcn-TreeSearch}\xspace}
\def\sv{S2V-DQN\xspace}

\newtheorem{thm}{\textbf{Theorem}}

\newtheorem{defn}{\textbf{Definition}}
\newtheorem{lem}{\textbf{Lemma}}

\newtheorem{proof}{Proof}

\title{Learning Heuristics over Large Graphs via Deep Reinforcement Learning}

\author{%
Sahil Manchanda, Akash Mittal\thanks{denotes equal contribution}, Anuj Dhawan$^{*}$  \\
   Indian Institute of Technology Delhi \\
   \texttt{\{sahil.manchanda,cs1150208,Anuj.Dhawan.cs115\}@cse.iitd.ac.in}\\
      \AND
  Sourav Medya$^1$, Sayan Ranu$^2$, Ambuj Singh$^3$ \\
   $^1$Northwestern University,$^2$Indian Institute of Technology Delhi\\   $^3$University of California Santa Barbara \\
   $^1$\texttt{sourav.medya@kellogg.northwestern.edu} \\
   $^2$\texttt{sayanranu@cse.iitd.ac.in,$^3$ambuj@ucsb.edu}\\
}

\begin{document}

\maketitle

\begin{abstract}
    There has been an increased interest in discovering heuristics for combinatorial problems on graphs through machine learning. While existing techniques have primarily focused on obtaining high-quality solutions, scalability to billion-sized graphs has not been adequately addressed. In addition, the impact of budget-constraint, which is necessary for many practical scenarios, remains to be studied. In this paper, we propose a framework called \gc to bridge these gaps.  
\gc trains a Graph Convolutional Network (GCN) using a novel \emph{probabilistic greedy} mechanism to predict the quality of a node.
To further facilitate the combinatorial nature of the problem, \gc utilizes a $Q$-learning framework, which is made efficient through \emph{importance sampling}.
We perform extensive experiments on real graphs to benchmark the efficiency and efficacy of \gc. Our results establish that \gc is $100$ times faster and marginally better in quality than state-of-the-art algorithms for learning combinatorial algorithms. Additionally, a case-study on the practical combinatorial problem of Influence Maximization (IM) shows \gc is $150$ times faster than the specialized IM algorithm IMM with similar quality. 
\end{abstract}

\section{Introduction and Related Work}
\label{sec:intro}
\vspace{-0.10in}

Combinatorial optimization problems on graphs appear routinely in various applications such as viral marketing in social networks \cite{kempe2003maximizing,edgeadd}, computational sustainability \cite{dilkina2011}, health-care \cite{Wilder2018aamas}, and infrastructure deployment~\cite{medya2018noticeable,netclus1,netclus2,infocom}. In these \emph{set combinatorial problems}, the goal is to identify the set of nodes that optimizes a given objective function. These optimization problems are often NP-hard. Therefore, designing an exact algorithm is infeasible and polynomial-time algorithms, with or without approximation guarantees, are often desired and used in practice \cite{jung2012irie,imm}. Furthermore, these graphs are often dynamic in nature and the approximation algorithms need to be run repeatedly at regular intervals. Since real-world graphs may contain millions of nodes and edges, this entire process becomes tedious and time-consuming. 

%

To provide a concrete example, consider the problem of viral marketing on social networks through \emph{Influence Maximization} \cite{benchmarking, kempe2003maximizing}. Given a budget $b$, the goal is to select $b$ nodes (users)  such that their endorsement of a certain product (ex: through a tweet) is expected to initiate a cascade that reaches the largest number of nodes in the graph. This problem is NP-hard~\cite{kempe2003maximizing}. Advertising through social networks is a common practice today and needs to solved repeatedly due to the graphs being dynamic in nature. Furthermore, even the greedy approximation algorithm does not scale to large graphs~\cite{benchmarking} resulting in a large body of research work \cite{imm,jung2012irie,celf,pmc,kempe2003maximizing,chen2010scalable,wang2012scalable,cohen2014sketch}.


At this juncture, we highlight two key observations. First, although the graph is changing, the underlying model generating the graph is likely to remain the same. Second, the nodes that get selected in the answer set of the approximation algorithm may have certain properties common in them. 
Motivated by these observations, we ask the following question \cite{dai2017learning}: \textit{Given a set combinatorial problem $P$ on graph $G$ and its corresponding solution set $S$, can we learn an approximation algorithm for problem $P$ and solve it on an unseen graph that is similar to $G$?} 
\vspace{-0.1in}

\subsection{Limitations of Existing Work}
\label{sec:relatedwork}
\vspace{-0.1in}

The above observations were first highlighted by \sv \cite{dai2017learning}, where they 
 show that it is indeed possible to \emph{learn} combinatorial algorithms on graphs. Subsequently, an improved approach was proposed in \gt \cite{li2018combinatorial}. Despite these efforts, there is scope for further improvement. 

$\bullet$ \textbf{Scalability}: The primary focus of both \gt and \sv have been on obtaining quality that is as close to the optimal as possible. Efficiency studies, however, are limited to graphs containing only hundreds of thousands nodes. To provide a concrete case study, we apply \gt for the Influence Maximization problem on the YouTube social network. We observe that \gt takes  one hour on a graph containing a million edges (Fig.~\ref{fig:time_nips}; we will revisit this experiment in \S~\ref{sec:im}). 
Real-life graphs may contain billions of edges (See. Table~\ref{table:data_description}). 

$\bullet$ \textbf{Generalizability to real-life combinatorial problems: } 
\gt proposes a learning-based heuristic for the Maximal Independent Set problem (MIS).  When the combinatorial problem is not MIS, \gt suggests that we map that problem to MIS. Consequently, for problems that are not easily mappable to MIS, the efficacy is compromised (ex: Influence Maximization). 

$\bullet$ \textbf{Budget constraints: } Both \gt and \sv solve the decision versions of combinatorial problems (Ex. set cover, vertex cover). 
In real life, we often encounter their budget-constrained versions, such as max-cover and Influence Maximization \cite{kempe2003maximizing}. 


 Among other related work, Gasse et al.~\cite{gasse2019exact} used GCN for learning branch-and-bound variable selection policies, whereas Prates et al. \cite{prates2019learning} focused on solving Travelling Salesman Problem. However, the proposed techniques in these papers do not directly apply to our setting of set combinatorial problems.


\vspace{-0.10in}
\subsection{Contributions} 
\label{sec:contributions}

At the core of our study lies the observation that although the graph may be large, only a small percentage of the nodes are likely to contribute to the solution set. Thus, pruning the search space is as important as prediction of the solution set. Both \sv \cite{dai2017learning} and \gt \cite{li2018combinatorial} have primarily focused on the prediction component. In particular, \sv learns an end-to-end neural model on the entire graph through reinforcement learning. The neural model integrates node embedding and $Q$-learning into a single integrated framework. Consequently, the model is bogged down by a large number of parameters, which needs to be learned on the entire node set. As a result, we will show in \S.~\ref{sec:experiments} that \sv fails to scale to graphs beyond $20,000$ nodes.

  On the other hand, \gt employs a two-component framework: \textbf{(1)} a graph convolutional network (GCN) to learn and predict the individual \emph{value} of each node, and \textbf{(2)} a \emph{tree-search} component to analyze the dependence among nodes and identify the solution set that collectively works well. Following tree-search, GCN is repeated on a reduced graph and this process continues iteratively. This approach is not scalable to large graphs since due to repeated iterations of GCN and TreeSearch where each iteration of tree-search has $O(|E|)$ complexity; $E$ is the set of edges. 
  
  Our method \gc builds on the observation that computationally expensive predictions should be attempted only for promising nodes. Towards that end, \gc has two separate components: \textbf{(1)} a GCN to prune \emph{poor} nodes and learn embeddings of \emph{good} nodes in a \emph{supervised} manner, and \textbf{(2)} a $Q$-learning component that focuses only on the \emph{good} nodes to predict the solution set. 
     Thus, unlike \sv, \gc uses a \emph{mixture} of supervised and reinforcement learning, and does not employ an end-to-end architecture. Consequently, the prediction framework is lightweight with a significantly reduced number of parameters. 
     
     When compared to \gt, although both techniques use a GCN, in \gc, we train using a novel \emph{probabilistic greedy} mechanism. Furthermore, instead of an iterative procedure of repeated GCN and TreeSearch calls, \gc performs a single forward pass through GCN during inference. In addition, unlike TreeSearch, which is specifically tailored for the MIS problem, \gc is problem-agnostic \footnote{We are, however, limited to set combinatorial problems only.}. Finally, unlike both \sv and \gt, \gc uses lightweight operations to prune \emph{poor} nodes and focus expensive computations only on nodes with a high potential of being part of the solution set. The pruning of the search space not only enhances scalability but also removes noise from the search space leading to improved prediction quality. Owing to these design choices, {\bf (1)} \gc is scalable to billion-sized graphs and up to $100$ times faster, {\bf (2)} on average, computes higher quality solution sets than \sv and \gt, and {\bf (3)} improves upon the state-of-the-art  algorithm for Influence Maximization on social networks.

\section{Problem Formulation}
\label{sec:formulation}


\textbf{Objective:} \textit{Given a budget-constrained set combinatorial problem $P$ over graphs drawn from distribution $D$, \emph{learn} a heuristic to solve problem $P$ on an unseen graph $G$ generated from $D$. }

Next, we describe three instances of  budget-constrained set combinatorial problems on graphs.

\textbf{Maximum Coverage Problem on bipartite graph (MCP):}
Given a bipartite graph $G=(V,E)$, where $V=A\cup B$, 
  and a budget $b$, find a set $S^*\subseteq A$ of $b$ nodes such that  coverage is maximized. The coverage of set $S^*$ is defined as $f(S^*)=\frac{|X|}{|B|}$, where $X=\{j|(i,j)\in E,i \in S^*, j\in B\}$.\\
\textbf{Budget-constrained Maximum Vertex Cover (MVC): }
Given a graph $G=(V,E)$ and a budget $b$, find a set $S^*$ of $b$ nodes such that the coverage $f(S^*)$ of $S^*$ is maximized. $f(S^*)=\frac{|X|}{|E|}$, where $X=\{(i,j)|(i,j)\in E,i \in S^*, j\in V\}$. 

\textbf{Influence Maximization (IM) \cite{benchmarking}: }
			Given a budget $b$, a social network $G$, and a information diffusion model $\mathcal{M}$, select a set $S^*$ of $b$ nodes such that the expected diffusion spread $f(S^*)=\expectation[\flow(S^*)]$ is maximized. (See App. A in supplementary for more details).
		      

\vspace{-0.05in}





\label{sec:gcomb}

\begin{figure}[t]
\centering
    \includegraphics[width=5in]{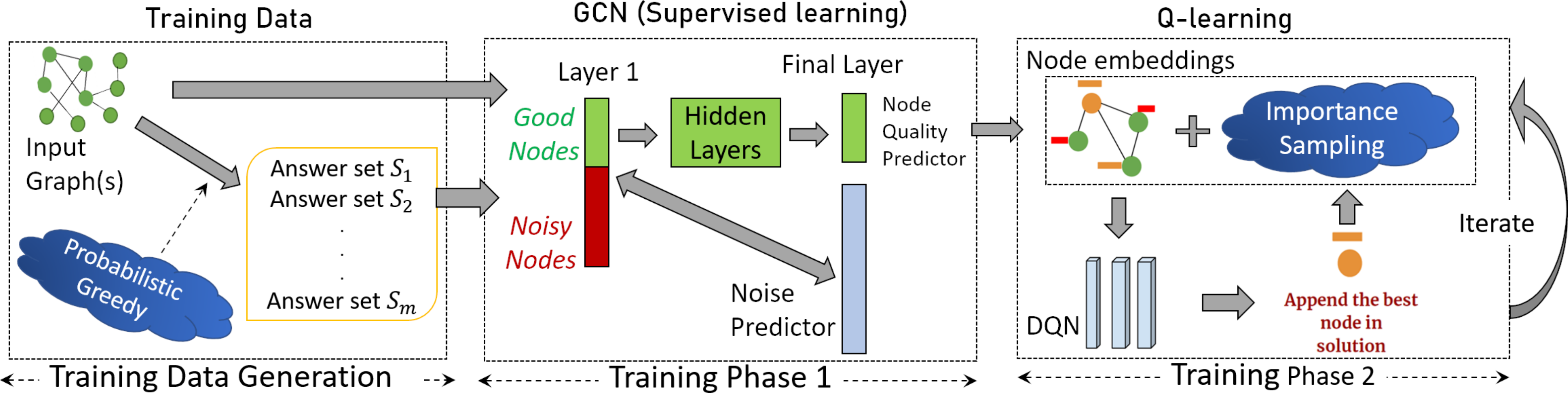}
    \caption{The flowchart of the training phase of \gc.}
     \label{fig:example1}
\end{figure}

\section{GCOMB}
\label{sec:gcomb_det}
\vspace{-0.10in}

The input to the training phase is a set of graphs and the optimization function $f(\cdot)$ corresponding to the combinatorial problem in hand. The output is a sequence of two separate neural graphs, GCN~\cite{hamilton2017inductive} and $Q$-learning network, with their corresponding learned parameters $\Theta_G$ and $\Theta_Q$ respectively. In the testing phase, the inputs include a graph $G=(V,E)$, the optimization function $f(\cdot)$ and the budget $b$. The  output of the testing part is the solution set of nodes constructed using the learned neural networks. 
Fig.~\ref{fig:example1} presents the training pipeline. We will now discuss each of the phases. 

\vspace{-0.10in}
\subsection{Generating Training Data for GCN}
\label{sec:gcn}



Our goal is to learn node embeddings that can predict ``quality'', and thereby, identify those nodes that are likely to be part of the answer set. 
We could adopt a classification-based method, where, given a training graph $G=(V,E)$, budget $b$ and its solution set $S$, a node $v$ is called \emph{positive} if $v\in S$; otherwise it is negative.
This approach, however, assumes all nodes that are not a part of $S$ to be equally bad. In reality, this may not be the case. Consider the case where $f(\{v_1\})$=$f(\{v_2\})$, but the marginal gain of node $v_2$ given $S=\{v_1\}$, i.e., $f(\{v_1,v_2\})-f(\{v_1\})$, is $0$ and vice versa. In this scenario, only one of $v_1$ and $v_2$ would be selected in the answer set although both are of equal quality on their own. 

\textbf{Probabilistic greedy:} To address the above issue, we \emph{sample} from the \emph{solution space} in a greedy manner and learn embeddings that reflect the \emph{marginal gain} $f(S\cup \{v\})-f(S)$ provided by a node $v$ towards the solution set $S$ (Alg. 2 in Appendix). To sample from the solution space, 
in each iteration, instead of selecting the node with the highest marginal gain, we choose a node with probability proportional to its marginal gain. The probabilistic greedy algorithm runs $m$ times to construct $m$ different solution sets $\mathbb{S}=\{S_1,\cdots,S_m\}$ and the score of node $v\in V$ is set to: 

\vspace{-0.10in}
\begin{equation}
\label{eq:score}
score(v)=\frac{\sum_i^m gain_i(v)}{\sum_i^m f(S_i)}
\end{equation}
\vspace{-0.10in}

Here, $gain_i(v)$ denotes the marginal gain contribution of $v$ to $S_i$. Specifically, assume $v$ is added to $S_i$ in the $(j+1)_{th}$ iteration and let $S_i^j$ be the set of nodes that were added in the first $j$ iterations while constructing $S_i$. Then, $gain_i(v)=f\left(S_i^j\cup \{v\}\right)-f\left(S_i^j\right)$. 
In our experiments, $m$ is set to 30 for all three problems of MCP, MVC and IM.

\textbf{Termination condition of probabilistic greedy: }
Probabilistic greedy runs till \emph{convergence} of the marginal gains, i.e., $gain_i(v)\leq \Delta$, where $\Delta$ is a small value. The goal here is to identify all nodes that could potentially be part of the solution set for \emph{any} given budget.   
$\Delta$ in our experiments is set to $0.01$ for all three problems of MCP, MVC and IM.  

 



\vspace{-0.10in} 
\subsection{Training the GCN} 
\vspace{-0.10in} 

Our goal in this phase is two-fold: \textbf{(1) } Identify nodes that are unlikely to be part of the solution set and are therefore noise in the context of our problem; {\bf (2)} Learn a predictive model for node quality.

\textbf{Noise predictor: } 
The noise predictor should be lightweight so that expensive computations are reserved only for the good  nodes. With this goal, we exploit the first layer information of the GCN and learn a classifier to predict for a given budget $b$, whether a node can be safely pruned without affecting the quality of the solution set. Typically, the first layer of a GCN contains the raw features of nodes that are relevant for the problem being solved. In \gc, we use the summation of the outgoing edge weights as node features. Let $\boldsymbol{x}_v$ denote the total outgoing edge weight of node $v$. 
To learn the noise predictor, given a set of training graphs $\{G_1,\cdots,G_t\}$, we first sort all nodes based on $\boldsymbol{x}_v$. Let $rank(v,G_i)$ denote the position of $v$ in the sorted sequence based on $\boldsymbol{x}_v$ in $G_i$. Furthermore, let $S^i_j$ denote the $j^{th}$ solution set constructed by probabilistic greedy on $G_i$. Given a budget $b$, $S^j_{G_i,b}\subseteq S^i_j$ denotes the subset containing the first $b$ nodes added to $S^i_j$ by probabilistic greedy.  
 Therefore, $r^b_{G_i}=\max_{j=0}^m\left\{\max_{\forall v\in S^j_{G_i,b}}\left\{rank(v,G_i)\right\}\right\}$ represents the lowest rank of any node in a solution set of budget $b$ in $G_i$. This measure is further generalized to all training graphs in the form of $r^b_{max}=\max_{\forall G_i}\left\{r^b_{G_i}\right\}$,  which represents the lowest rank of any node that has a realistic chance of being included in an answer set of budget $b$. 
To generalize across budgets, we compute $r^{b_i}_{max}$ for a series of budgets $\{b_1,\cdots,b_{max}\}$, where  $b_{max}=\max_{\forall G_i}\left\{\max_{j=0}^m\left\{|S^i_j|\right\}\right\}$. On this data, we can perform curve fitting~\cite{curve} to predict $r^{b}_{max}$ for any (unseen) budget $b$. In our experiments, we use linear interpolation. To generalize across graph sizes, all of the above computations are performed on \emph{normalized} budgets, where $b$ is expressed in terms of the proportion of nodes with respect to the node set size of the graph. Similarly, rank $rank(v,G_i)$ is expressed in terms of percentile.

\label{sec:Training the GCN}
\textbf{Node quality predictor: }To train the GCN, we sample a training graph $G_i=(V_i,E_i)$ and a (normalized) budget $b$ from the range $(0,b^i_{max}]$, where $b^i_{max}=\max_{j=0}^m\left\{\frac{|S^i_j|}{|V_i|}\right\}$. This tuple is sent to the noise predictor to obtain the good (non-noisy) nodes. The GCN parameters ($\Theta_G$) are next learned by minimizing the loss function only on the good nodes. Specifically, for each good node $v$, we want to learn embeddings that can predict $score(v)$ through a surrogate function $score'(v)$. Towards that end, we draw multiple samples of training graphs and budgets, and the parameters are learned by minimizing the \emph{mean squared error} loss (See Alg.3 for detailed pseudocode in the Supplementary).  
 

\vspace{-0.05in}
\begin{equation}
J(\Theta_G)=\sum_{\sim\langle G_i, b \rangle}\frac{1}{|V^g_i|}\sum_{\forall v\in V^g_i}(score(v)-score'(v))^2
\end{equation}
In the above equation,  $V^g_i$ denotes the set of good nodes for budget $b$ in graph $G_i$. Since GCNs are trained through message passing, in a GCN with $K$ hidden layers, the computation graph is limited to the induced subgraph formed by the $K$-hop neighbors of $V_i^g$, instead of the entire graph.

\vspace{-0.10in}
\subsection{Learning $Q$-function}
\label{sec:qlearning}
\vspace{-0.10in}

While GCN captures the individual importance of a node, $Q$-learning \cite{sutton2018reinforcement} learns the combinatorial aspect in a budget-independent manner. Given a set of nodes $S$ and a node $v\not\in S$, we predict the $n$-step reward, $Q_n(S,v)$, for adding $v$ to set $S$ (action) via the surrogate function $Q'_n(S,v;\Theta_Q)$. 

\textbf{Defining the framework:} 
We define the $Q$-learning task in terms of state space, action, reward, policy and termination with the input as a set of nodes and their predicted scores.
    
$\bullet$ \textbf{State space: } The state space characterizes the state of the system at any time step $t$ in terms of the candidate nodes being considered, i.e., $C_t=V^g\setminus S_t$, with respect to the partially computed solution set $S_t$; $V^g$ represents the set of good nodes from a training graph. In a combinatorial problem over nodes, two factors have a strong influence: \textbf{(1)} the individual quality of a node, and \textbf{(2)} its \emph{locality}. The quality of a node $v$ is captured through $score'(v)$. Locality is an important factor since two high-quality nodes from the same neighborhood may not be good collectively. 
The locality of a node $v\in C_t$ ($C_t=V^g\setminus S_t$) is defined as:
\begin{equation}
\label{eq:loc}
loc(v,S_t)=|N(v)\setminus \cup_{\forall u\in S_t} N(u)|
\end{equation}

where $N(v)=\{v'\in V \mid (v,v')\in E\}$ are the neighbors of $v$. Note that  $N(v)$ may contain noisy nodes since they contribute to the locality of $v\in V^g$. However, locality (and $q$-learning in  general) is computed only on good nodes. 
The initial representation $\boldsymbol{\mu_v}$ of each node $v\in C_t$ is therefore the $2$-dimensional vector $[score'(v),loc(v,S_t)]$. The representation of the set of nodes $C_t$ is defined as $\boldsymbol{\mu_{C_t}}=\textsc{MaxPool}\left\{\mu_v \mid\: v\in C_t\right\}$. $\boldsymbol{\mu_{S_t}}$ is defined analogously as well. We use \textsc{MaxPool} since it captures the best available candidate node better than alternatives such as \textsc{MeanPool}. Empirically, we obtain better results as well.



    
$\bullet$ \textbf{Action and Reward: }An action corresponds to adding a node $v\in C_t$ to the solution set $S_t$. The immediate ($0$-step) reward of the action is its marginal gain, i.e. $r(S_t,v)=f(S_t\cup\{v\})-f(S_t)$.

$\bullet$ \textbf{Policy and Termination: }The policy $\pi(v\mid S_t)$ selects the 
node with the highest \emph{predicted} $n$-step reward, i.e., $\arg\max_{v \in C_t} Q'_n(S_t,v;\Theta_Q)$. 
We terminate after training the model for $T$  samples.


\textbf{Learning the parameter set $\Theta_Q$: } 
We partition $\Theta_Q$ into three weight matrices $\Theta_1$, $\Theta_2$, $\Theta_3$, and one weight vector $\Theta_4$ such that, $Q'_n(S_t,v;\Theta_Q) = \Theta_4\cdot \boldsymbol{\mu_{C_t,S_t,v}}$, where $\boldsymbol{\mu_{C_t,S_t,v}} =\textsc{Concat} \left (\Theta_1 \cdot\boldsymbol{\mu}_{C_t} , \Theta_2 \cdot\boldsymbol{\mu}_{S_t},\Theta_3 \cdot\boldsymbol{\mu}_{v}\right ) $. If we want to encode the state space in a $d$-dimensional layer, the dimensions of the weight vectors are as follows: $\Theta_4 \in \mathbb{R}^{1\times 3d}; \Theta_1,\Theta_2,\Theta_3 \in \mathbb{R}^{d\times 2}$. $Q$-learning updates parameters in a single episode via Adam optimizer\cite{adam} to minimize the squared loss.
\vspace{-0.15in}
\begin{equation}
\nonumber
J(\Theta_Q)=(y - Q'_n(S_t, v_t; \Theta_Q))^2
\text{, where } y=\gamma \cdot  \max_{v\in V^g}\{Q'_n(S_{t+n},v;\Theta_Q)\}+ \sum_{i=0}^{n-1}r(S_{t+i},v_{t+i})
\end{equation}
\vspace{-0.15in}
 
$\gamma$ is the \emph{discount factor} and balances the importance of immediate reward with the predicted $n$-step future reward~\cite{sutton2018reinforcement}. 
 The pseudocode with more details is provided in the Supplementary (App. C).
 
 \vspace{-0.10in}
 \subsubsection{Importance Sampling for Fast Locality Computation}
\label{sec:locality}
\vspace{-0.10in}

Since degrees of nodes in real graphs may be very high, computing locality (Eq.~\ref{eq:loc}) is expensive. Furthermore, locality is re-computed in each iteration. We negate this computational bottleneck through \emph{importance sampling}. 
Let $N(V^g)=\{\left( v,u \right)\in E\mid v\in V^g\}$ be the neighbors of all nodes in $V^g$. Given a sample size $z$, we extract a subset $N_z(V^g)\subseteq N(V^g)$ of size $z$ and compute locality only based on the nodes in $N_z(V^g)$. Importance sampling samples elements proportional to their importance. The \emph{importance} of a node in $N(V^g)$ is defined as $I(v)=\frac{score'(v)}{\sum_{\forall v'\in N(V^g)}score'(v')}$.

\textbf{Determining sample size: }Let $\mu_{N(V^g)}$ be the mean importance of all nodes in $N(V^g)$ and $\hat{\mu}_{N_z(V^g)}$ the mean importance of sampled nodes. The sampling is \emph{accurate} if $\mu_{N(V^g)}\approx \hat{\mu}_{N_z(V^g)}$. 
\vspace{-0.10in}
\begin{thm}
\label{thm:samplesize}
 Given an error bound $\epsilon$, if sample size $z$ is $O\left(\frac{\log |N(V^g)|}{\epsilon^2}\right)$, then $P\left[|\hat{\mu}_{N_z(V^g)} - \mu_{N(V^g)}|< \epsilon\right] >1-\frac{1}{|N(V^g)|^2}$.
\end{thm}
\vspace{-0.10in}

\textbf{Remarks: } {\bf (1)} The sample size grows \emph{logarithmically} with the neighborhood size, i.e., $|N(V^g)|$ and thus scalable to large graphs. {\bf (2)} $z$ is an inversely proportional function of the error bound $\epsilon$. 




\vspace{-0.10in}
\subsection{Test Phase}
\label{Time complexity}
\vspace{-0.10in}
Given an unseen graph $G$ and budget $b$, we {\bf (1)} identify the noisy nodes, {\bf (2)} embed good nodes through a single forward pass through GCN, and {\bf (3)} use GCN output to embed them and perform Q-learning to compute the final solution set.

\textbf{Complexity analysis:} The time complexity of the test phase in \gc is $O\left(|V|+|V^{g,K}|\left(dm_G+m_G^2\right)+|V^g|b\left(d+m_Q\right)\right)$, where $d$ is the average degree of a node, $m_G$ and $m_Q$ are the dimensions of the embeddings in GCN and $Q$-learning respectively, $K$ is the number of layers in GCN, and $V^{g,K}$ represents the set of nodes within the $K$-hop neighborhood of $V^g$. The space complexity is $O(|V|+|E| + Km_G^2 + m_Q)$. The derivations are provided in App.~D.
\vspace{-0.1in}
\section{Empirical Evaluation}
\label{sec:experiments}
\vspace{-0.1in}

In this section, we benchmark \gc against \gt and \sv, and establish that \gc produces marginally improved quality, while being orders of magnitudes faster. The source code can be found at \url{https://github.com/idea-iitd/GCOMB} .
\vspace{-0.1in}

\begin{table}[t]
\centering
\vspace{-0.10in}
\subfloat[Datasets from SNAP repository~\cite{snap}.]{
\scalebox{0.75}{
\label{table:data_description}
\begin{tabular}{|c | c | c | c | c | c |}
\hline
\textbf{Name}& \textbf{$|V|$} & \textbf{$|E|$}\\
\hline
Brightkite (BK)  & 58.2K & 214K\\
\hline
Twitter-ego (TW-ew)  & 81.3K & 1.7M\\
\hline
Gowalla (GO)&  196.5K & 950.3K\\
\hline
YouTube (YT) & 1.13M & 2.99M\\
\hline
StackOverflow (Stack) & 2.69M  & 5.9M\\
\hline
Orkut & 3.07M & 117.1M\\
\hline
Twitter (TW) & 41.6M & 1.5B\\
\hline
FriendSter (FS) & 65.6M & 1.8B\\
\hline
\end{tabular}}}
\subfloat[Coverage in MCP]{
\scalebox{0.70}{
\label{tab:optimal}
\begin{tabular}{| c||c |c | c || c |c|c|}
\hline
\textbf{Dataset}& \multicolumn{3}{c}{\bf BP-500}& \multicolumn{3}{|c|}{\bf Gowalla-900}\\
\hline
\textbf{Budget}& \textbf{\gc} & \textbf{Greedy} & \textbf{Optimal} & \textbf{\gc} & \textbf{Greedy} & \textbf{Optimal}\\
\hline
\textbf{2}& $\textbf{0.295}$ & $0.295$ & $0.295$ & $\textbf{0.75}$ & $0.75$ & $0.75$\\
\hline
\textbf{4}& $0.495$ & $0.505$ & $0.51$&  $0.902$ & $0.904$ & $0.904$\\
\hline
\textbf{6}& $0.765$ & $0.77$ & $0.773$&$\textbf{0.941}$ & $0.93$ & $0.941$\\
\hline
\textbf{10}& $0.843$ & $0.845$ & $0.845$& $\textbf{0.952}$ & $0.952$ & $0.952$\\
\hline
\textbf{15}& $\textbf{0.96}$ & $0.953$ & $0.963$& $\textbf{0.963}$ & $0.963$ & $0.963$\\
\hline
\textbf{20}& $\textbf{0.998}$ & $0.99$ & $1$& $\textbf{0.974}$ & $0.974$& $0.974$\\
\textbf{25}& $\textbf{1}$ & $1$ & $1$ & $\textbf{0.985}$ & $0.985$ & $0.985$\\
\textbf{30}& $-$ & $-$ & $-$& $\textbf{0.996}$ & $0.996$ & $0.996$\\
\textbf{35}& $-$ & $-$ & $-$& $\textbf{1}$ & $1$ & $1$ \\
\hline
\end{tabular}}
}
\caption{In (b), the specific cases where \gc matches or outperforms Greedy are highlighted in bold. Gowalla-900 is a small subgraph of 900 nodes extracted from Gowalla (See App. I for details). }
 \end{table}
 
\subsection{Experimental Setup}
\label{sec:setup}

All experiments are performed on a machine running Intel Xeon E5-2698v4 processor with $64$ cores, having $1$ Nvidia $1080$ Ti GPU card with 12GB GPU memory, and $256$ GB RAM with Ubuntu $16.04$. 
 All experiments are repeated $5$ times and we report the average of the metric being measured.

\textbf{Datasets: }
 Table~\ref{table:data_description}) lists the real datasets used for our experiments.\\  
 \textit{Random Bipartite Graphs (BP):} We also use the synthetic random bipartite graphs from \sv \cite{dai2017learning}. In this model, given the number of nodes, they are partitioned into two sets with $20\%$ nodes in one side and the rest in other. The edge between any pair of nodes from different partitions is generated with probability $0.1$. We use BP-$X$ to denote a generated bipartite graph of $X$ nodes.



\textbf{Problem Instances:} The performance of \gc is benchmarked on Influence Maximization (IM), Maximum Vertex Cover (MVC), and Maximum Coverage Problem (MCP) (\S~\ref{sec:formulation}). 
Since MVC can be mapped to MCP, empirical results on MVC are included in App. M. 

\textbf{Baselines:} The performance of \gc is primarily compared with {\bf (1)} \emph{\gt} \cite{li2018combinatorial}, which is the state-of-the-art technique to learn combinatorial algorithms. In addition, for MCP, we also compare the performance with {\bf (2)} \emph{Greedy} (Alg.1 in App. B), {\bf (3)} \sv \cite{dai2017learning}, { \bf(5)} \emph{CELF} ~\cite{leskovec2007cost} and {\bf (6)} the \emph{Optimal} solution set  (obtained using CPLEX~\cite{cplex} on small datasets). Greedy and \textsc{Celf} guarantees a $1-1/e$ approximation for all three problems. We also compare with {\bf(6)} \emph{Stochastic Greedy(SG)}  ~\cite{mirzasoleiman2014lazier} in App. L. For the problem of IM, we also compare with the state-of-the-art algorithm {\bf (7)} \emph{IMM}~\cite{imm}. Additionally, we also compare \gc with { \bf(8)} \emph{OPIM}~\cite{tang2018online}. 
 For \sv, \gt, IMM, and OPIM we use the code shared by the authors. 

\textbf{Training: }In all our experiments, for a fair comparison of \gc with \sv and \gt, we train all models for $12$ hours and the best performing model on the validation set is used for inference. Nonetheless, we precisely measure the impact of training time in Fig.~\ref{fig:traintime}. The break-up  of time spent in each of the three training phases is shown in App. G in the Supplementary. 

\textbf{Parameters: }The parameters used for \gc are outlined in App. H and their impact on performance is analyzed in App. N. For \sv and \gt, the best performing parameter values are identified using grid-search. In IMM, we set $\epsilon=0.5$ as suggested by the authors. In OPIM, $\epsilon$ is recommended to be kept in range $[0.01,0.1]$. Thus, we set it to $\epsilon=0.05$. 

\vspace{-0.10in}
\subsection{Performance on Max Cover (MCP)}
\vspace{-0.10in}

We evaluate the methods on both synthetic random bipartite (BP) graphs as well as real networks. 
\textbf{Train-Validation-Test split:} While testing on any synthetic BP graph, we train and validate on five BP-1k graphs each. For real graphs, we train and validate on BrightKite (BK) ($50:50$ split for train and validate) and test on other real networks.  Since our real graphs are not bipartite, we convert it to one by making two copies of $V$: $V_1$ and $V_2$. We add an edge from $u\in V_1$ to $u'\in V_2$ if $(u,u')\in E$.


\textbf{Comparison with Greedy and Optimal: } Table~\ref{tab:optimal} presents the achieved coverage (Recall \S~\ref{sec:formulation} for definition of coverage). We note that Greedy provides an empirical approximation ratio of at least $99\%$ when compared to the optimal. This indicates that in larger datasets where we are unable to compute the optimal, Greedy can be assumed to be sufficiently close to the optimal. Second, \gc is sometimes able to perform even better than greedy. This indicates that $Q$-learning is able to learn a more generalized policy through \emph{delayed} rewards and avoid a myopic view of the solution space.

\textbf{Synthetic Datasets:} 
Table \ref{table:msc_gcomb_soa} presents the results. \gc and Greedy achieves the highest coverage consistently. While \sv performs marginally better than \gt, \sv is the least scalable among all techniques; it runs out of memory on graphs containing more than $20,000$ nodes. As discussed in details in \S~\ref{sec:contributions}, the non-scalability of \sv stems from relying on an architecture with significantly larger parameter set than \gc or \gt. 
In contrast, \gc avoids noisy nodes, and focuses the search operation only on the good nodes.
\begin{table}[t]
\subfloat[Coverage achieved in MCP at $b=15$.]{
\scalebox{0.7}{
\label{table:msc_gcomb_soa}
\begin{tabular}{|c|p{0.7in} |p{0.7in} | c | c |}
\hline
\textbf{Graph}& \textbf{\sv} & \textbf{GCN-TS} & \textbf{\gc} & \textbf{Greedy}\\
\hline
\textbf{BP-2k}& $0.87$ & $0.86$ & $\mathbf{0.89}$ & $0.89$\\
\hline
\textbf{BP-5k}& $0.85$ & $0.84$ & $\mathbf{0.86}$ & $0.86$\\
\hline
\textbf{BP-10k}& $0.84$ & $0.83$ & $\mathbf{0.85}$ & $0.85$\\
\hline
\textbf{BP-20k}& NA & $0.82$ & $\mathbf{0.83}$ & $0.83$\\
\hline
\end{tabular}}}\hspace{0.5in}
\subfloat[Speed-up against \textsc{CELF} in MCP on YT.]{
\scalebox{0.7}{
\label{tab:mcp_gcomb_celf}
\begin{tabular}{|c|p{0.9in}|}
\hline
\textbf{Budget}& \textbf{Speed-up} \\
\hline
\textbf{20}& $4$\\
\hline
\textbf{50}& $3.9$\\
\hline
\textbf{100}& $3.01$\\
\hline
\textbf{150}& $2.11$ \\
\hline
\textbf{200}& $2.01$ \\
\hline
\end{tabular}}}
\caption{ (a) Coverage on Random Graphs in MCP. (b) Speed-up achieved by \gc against \textsc{Celf} on YT in MCP.}
\vspace{-0.10in}
 \end{table}
 
 \textit{Impact of training time: } A complex model with more number of parameters results in slower learning. In Fig.~\ref{fig:traintime}, we measure the coverage against the training time. While \gc's performance saturates within $10$ minutes, \sv and \gt need $9$ and $5$ hours respectively for training to obtain its best performance.

 \textbf{Real Datasets: } 
 Figs.~\ref{fig:twt-e_quality} and \ref{fig:yt_quality} present the achieved Coverage as the budget is varied. \gc achieves similar quality as Greedy, while \gt is marginally inferior. The real impact of \gc is highlighted in Figs.~\ref{fig:twt-e_time} and \ref{fig:yt_time}, which shows that \gc is up to $2$ orders of magnitude faster than \gt and $10$ times faster than Greedy. Similar conclusion can also be drawn from the results on Gowalla dataset in App. K in Supplementary.

 \textbf{Comparison with \textsc{Celf}: } Table ~\ref{tab:mcp_gcomb_celf} presents the speed-up achieved by \gc against \textsc{Celf}. The first pass of \textsc{Celf} involves sorting the nodes, which has complexity $O(|V| log   |V|)$. On the other hand, no such sorting is required in \gc. Thus, the speed-up achieved is higher in smaller budgets.

\begin{figure*}[b]
	\centering
\vspace{-0.1in}
\subfloat[Bipartite graph (5k)]{
	\label{fig:traintime}
	\includegraphics[width=1.08in]{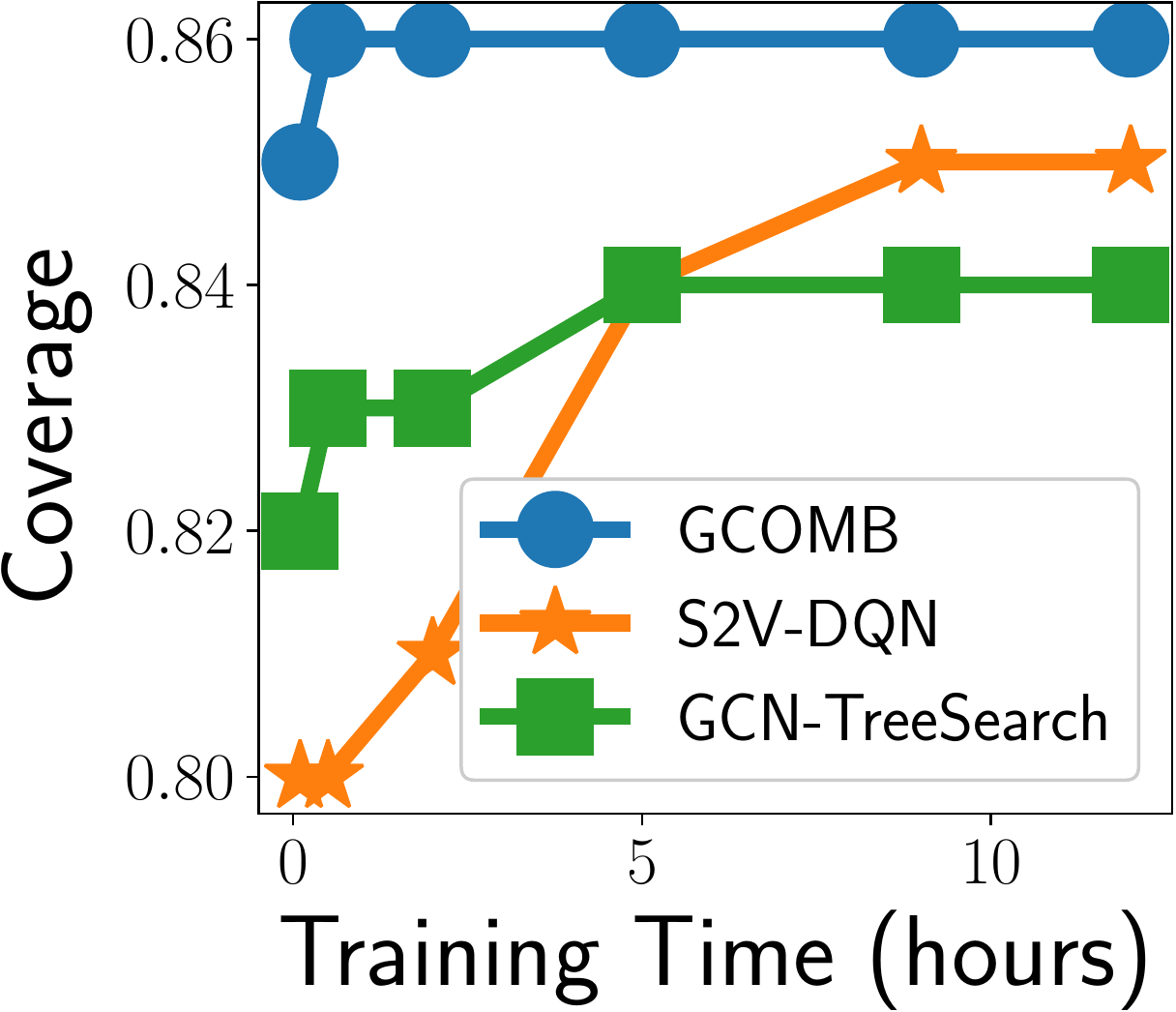}}
\subfloat[Twitter-ego]{
	\label{fig:twt-e_quality}
	\includegraphics[width=1.06in]{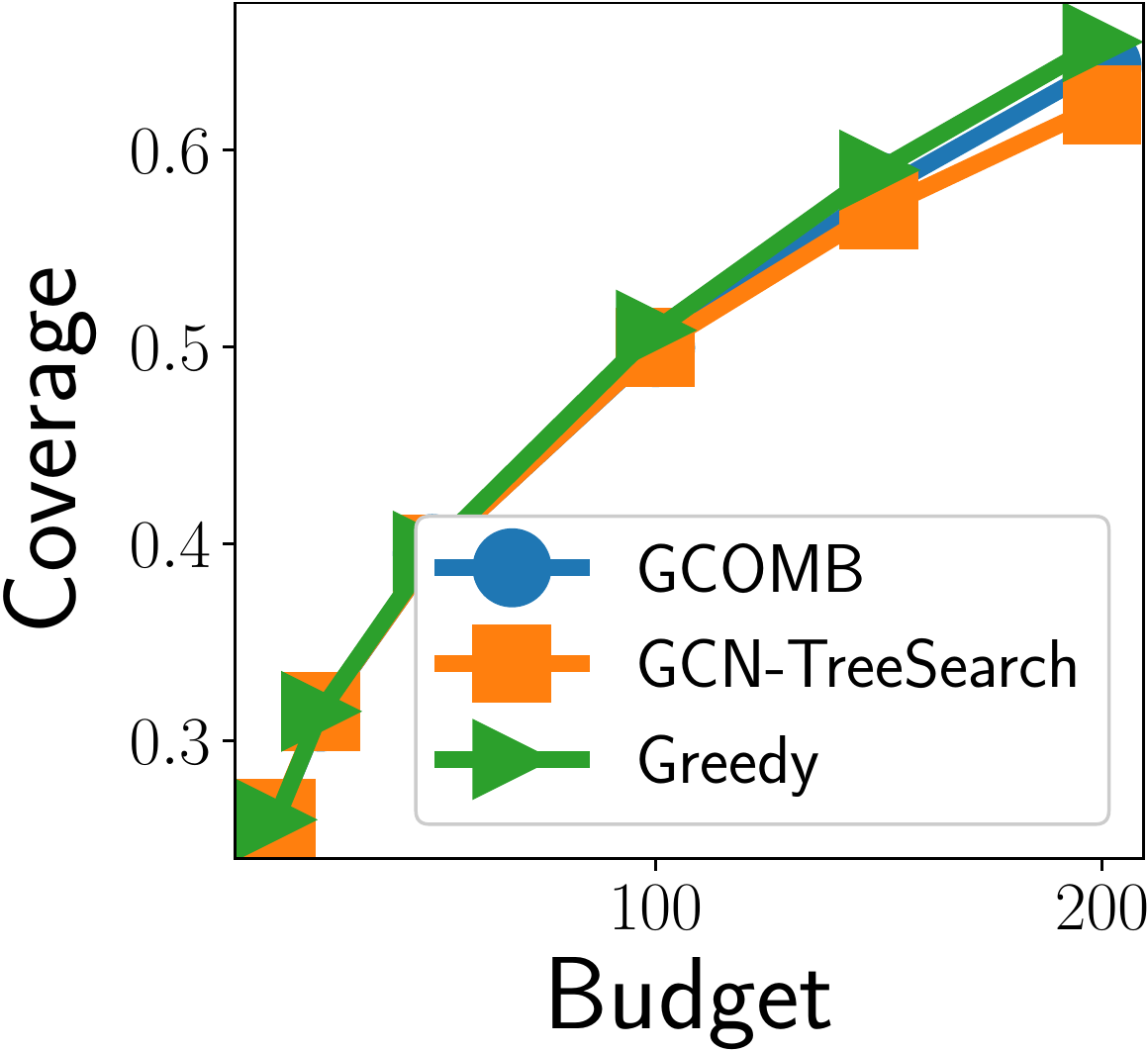}}
\subfloat[YouTube]{
	\label{fig:yt_quality}
	\includegraphics[width=1.10in]{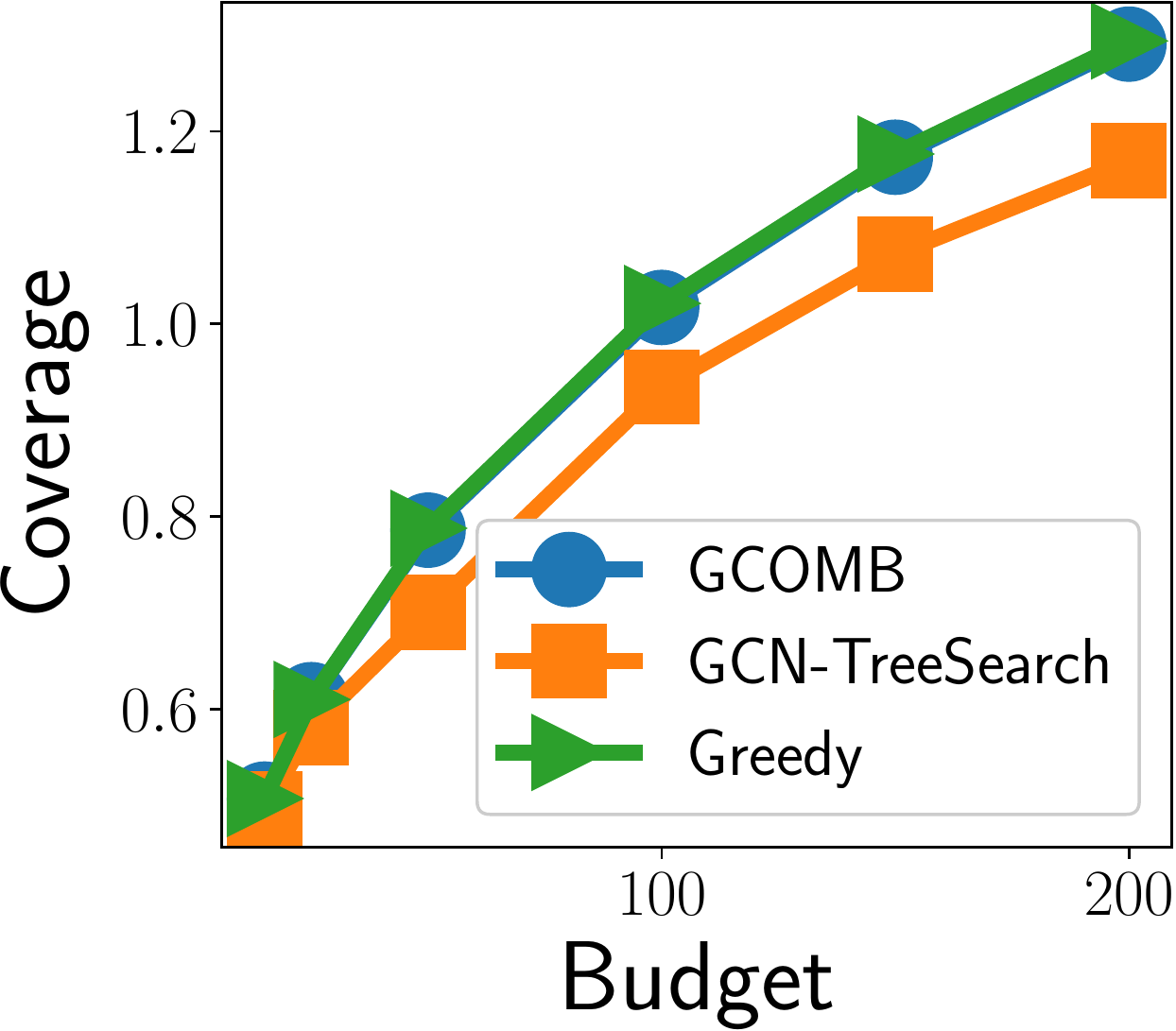}}
\subfloat[Twitter-ego]{
	\label{fig:twt-e_time}
	\includegraphics[width=1.05in,height=1.01in]{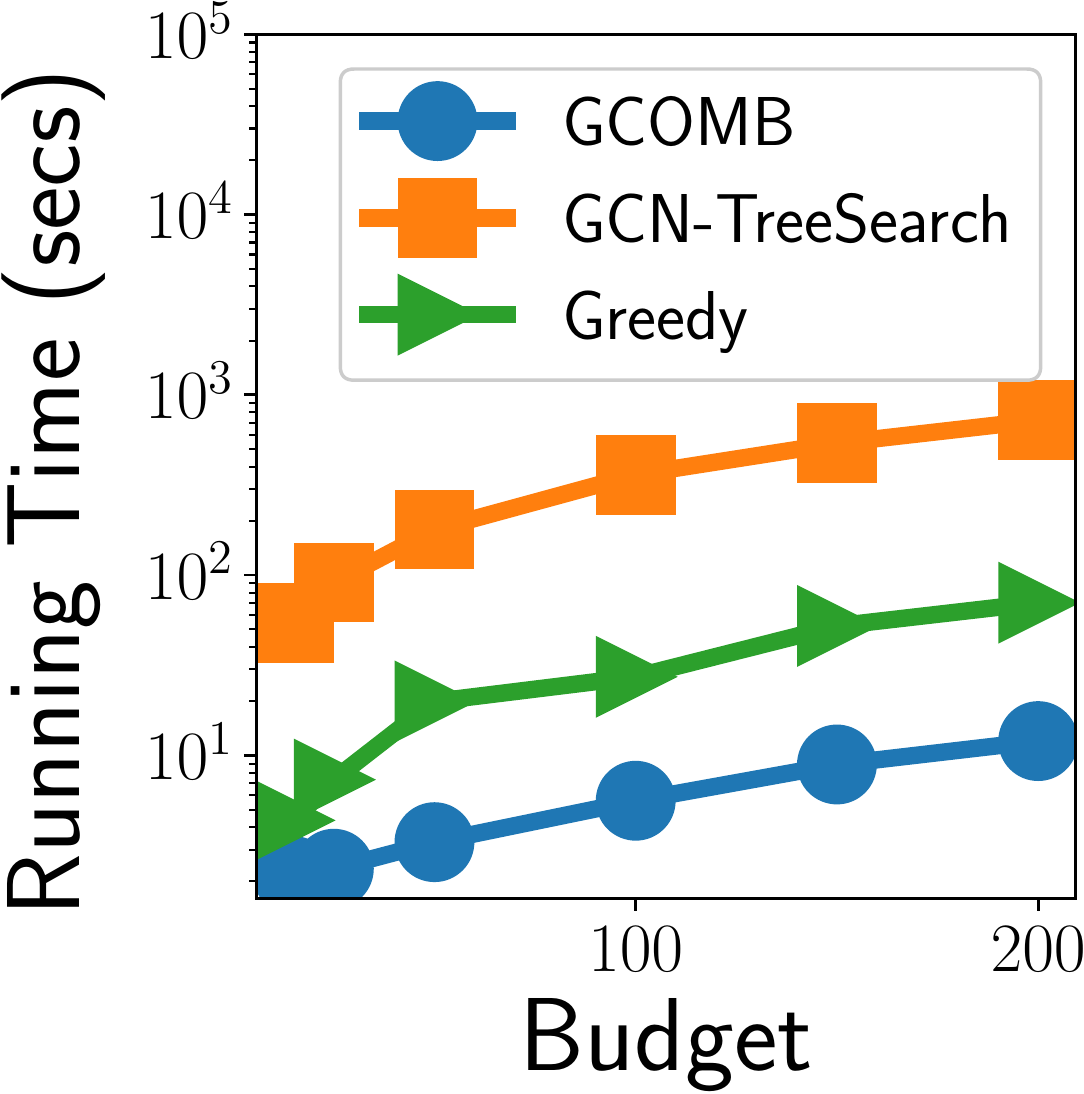}}
\subfloat[YouTube]{
	\label{fig:yt_time}
	\includegraphics[width=1.04in]{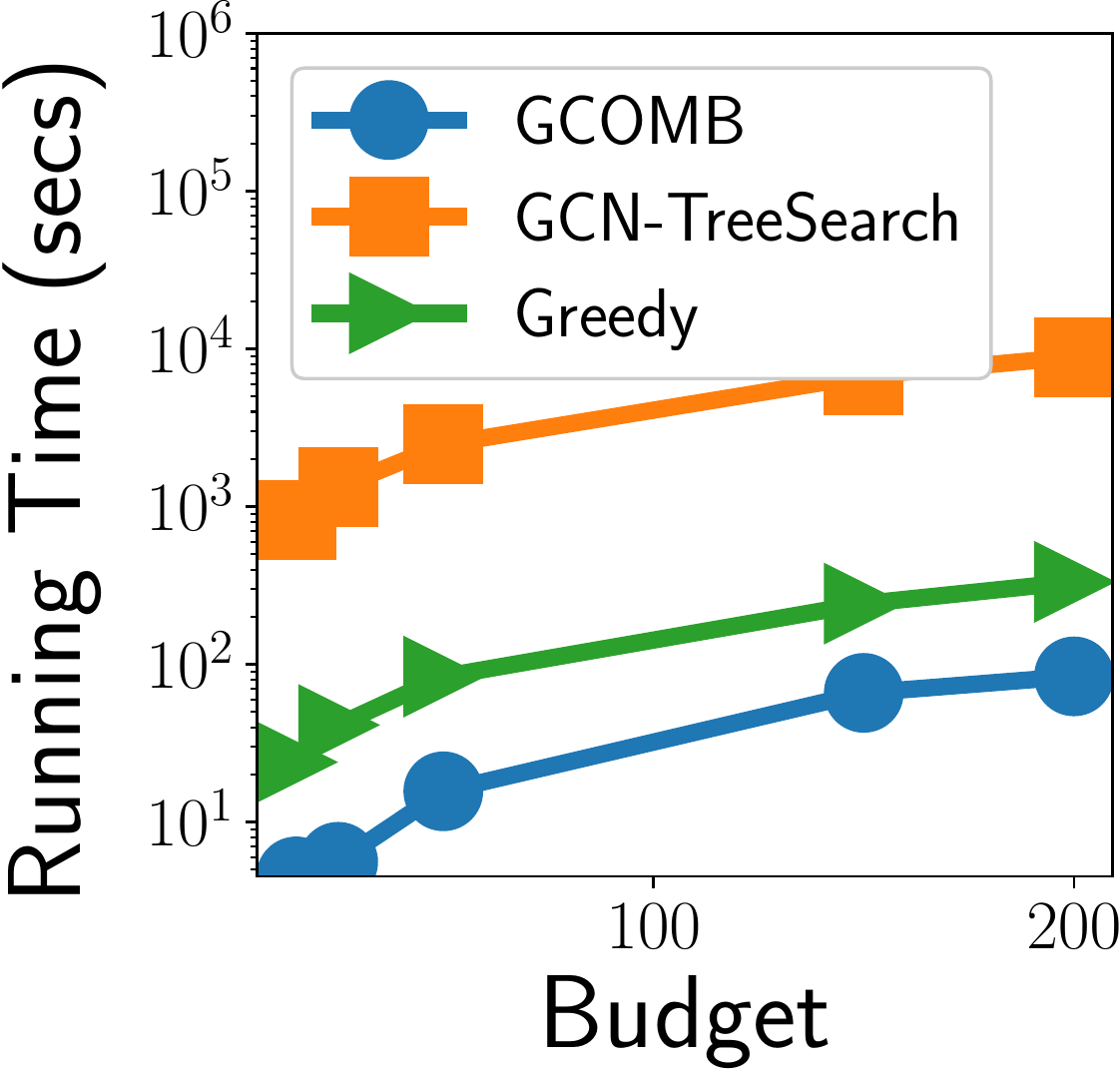}}
	\caption{MCP: (a) Improvement in Coverage against training time at $b=15$. (b-c) Coverage achieved by \gc, \gt and Greedy. (d-e) Running times in TW-ew and YT.}
\end{figure*}

\subsection{Performance on Influence Maximization}
\label{sec:im}

Influence Maximization (IM) is the hardest of the three combinatorial problems since estimating the spread of a node is \#P-hard \cite{kempe2003maximizing}. 

\textbf{Edge weights: } 
We assign edge weights that denote the influence of a connection using the two popular models \cite{benchmarking}: \textbf{(1) Constant (CO:)} All edge weights are set to $0.1$, \textbf{(2) Tri-valency (TV): } Edge weights are sampled randomly from the set $\{0.1, 0.01, 0.001\}$. In addition, we also employ a third \textbf{(3) Learned (LND)} model, where we \emph{learn} the influence probabilities from the action logs of users. This is only applicable to the Stack data which contain \emph{action logs} from $8/2008$ to $3/2016$. 
We define the influence of $u$ on $v$ as the probability of $v$ interacting with $u$'s content at least once in a month. 


\textbf{Train-Validation-Test split: } In all of the subsequent experiments, for CO and TV edge weight models, we train and validate on a subgraph sampled out of YT by randomly selecting $30\%$ of the edges ($50\%$ of this subset is used for training and $50\%$ is used for validation). For LND edge weight models, we train and validate on the subgraph induced by the $30\%$ of the earliest edges from Stack in terms of temporal order. While testing, on YT and Stack, we use the graph formed by the remaining $70\%$ of the edges that are not used for training. On other datasets, we use the entire graph for testing since neither those datasets nor their subsets are used for training purposes.  

\textbf{\gc vs.\gt: }Fig.~\ref{fig:time_nips} compares the running time in IM on progressively larger subgraphs extracted from YT. While \gt consumes $\approx 3$ hours on the $70\%$ sub-graph, \gc finishes in $5$ seconds. 

\textbf{\gc vs. \textsc{NoisePruner+Celf}}
\textsc{NoisePruner+Celf}, i.e., running \textsc{Celf} only on non-noisy nodes, is orders of magnitude slower than \gc in IM (See Fig~\ref{fig:celf_noise}). Pruning noisy nodes does not reduce the graph size; it only reduces the number of candidate nodes. To compute expected spread in IM, we still require the entire graph, resulting in non-scalability. 

 \textbf{Billion-sized graphs:}  IMM crashes on both the billion-sized datasets of TW and FS, as well as Orkut. Unsurprisingly, similar results have been reported in \cite{benchmarking}. IMM strategically samples a subgraph of the entire graph based on the edge weights. On this sampled subgraph, it estimates the influence of a node using \emph{reverse reachability sets}. 
On large graphs, the sample size exceeds the RAM capacity of 256GB. Hence, it crashes.
 In contrast, \gc finishes within minutes for smaller budgets ($b<30$) and within $100$ minutes on larger budgets of $100$ and $200$ (Figs.~\ref{fig:time_ic}-\ref{fig:time_tv} ). This massive scalability of \gc is a result of low storage overhead (only the graph and GCN and Q-learning parameters; detailed Space complexity provided in App. D in the Supplementary) and relying on just forwarded passes through GCN and $Q$-learning. The speed-up with respect to OPIM on billion-sized graphs can be seen in App. J. 

\textbf{Performance on YT and Stack:} Since IMM crashes on Orkut, TW, and FS, we compare the quality of \gc with IMM on YT and Stack. Table~\ref{tab:sd_imm} reports the results in terms of \emph{spread difference}, where $\text{Spread Difference}=\frac{f(S_{IMM})-f(S_{\gc})}{f(S_{IMM})}\times 100$.   $S_{IMM}$ and $S_{\gc}$ are answer sets computed by IMM and \gc respectively. A negative spread difference indicates better performance by \gc. The expected spread of a given set of nodes $S$, i.e. $f(S)$, is computed by taking the average spread across $10,000$ Monte Carlo simulations. 

 Table~\ref{tab:sd_imm} shows that the expected spread obtained by both techniques are extremely close. 
 The true impact of \gc is realized when Table~\ref{tab:sd_imm} is considered in conjunction with Figs.~\ref{fig:speedup_yt}-\ref{fig:speedup_st}, which shows \gc is $30$ to $160$ times faster than IMM. In this plot, speed-up is measured as $\frac{time_{IMM}}{time_{\gc}}$ where $time_{IMM}$ and $time_{\gc}$ are the running times of IMM and \gc respectively. 
 
 Similar behavior is observed when compared against OPIM as seen in Table ~\ref{tab:sd_opim} and Figs.~\ref{fig:speedup_yt_opim}-~\ref{fig:speedup_st_opim}.



\begin{table}[h]
\centering
\subfloat[Spread difference between IMM and \gc.]{
\scalebox{0.6}{
\label{tab:sd_imm}
\begin{tabular}{| c || c |c || c | c |c | }
\hline
\textbf{$b$}&  \textbf{YT-TV} & \textbf{YT-CO} &  \textbf{Stack-TV} & \textbf{Stack-CO} & {\bf Stack-LND}\\
\hline
\textbf{10}&  $\mathbf{-1\times 10^{-3}}$ & $1\times 10^{-4}$& $2\times 10^{-5}$& $\approx 0$& $1\times 10^{-5}$\\
\hline
\textbf{20}&  $\mathbf{-2\times 10^{-3}}$ & $2\times 10^{-4}$& $3\times 10^{-5}$& $3\times 10^{-5}$& $\mathbf{-7\times 10^{-5}}$\\
\hline
\textbf{50}&  $\mathbf{-3\times 10^{-3}}$ & $\mathbf{-5\times 10^{-5}}$ & $2\times 10^{-5}$ & $6\times 10^{-5}$& $\mathbf{-7\times 10^{-5}}$\\
\hline
\textbf{100}&  $\mathbf{-1\times 10^{-3}}$ & $6\times 10^{-4}$& $2\times 10^{-4}$& $2\times 10^{-4}$& $\mathbf{-1\times 10^{-4}}$\\
\hline
\textbf{150}& $\mathbf{-6\times 10^{-4}}$ & $3\times 10^{-4}$& $1\times 10^{-4}$& $1\times 10^{-4}$& $\mathbf{-3\times 10^{-5}}$\\
\hline
\textbf{200} & $\mathbf{-2\times 10^{-3}}$ & $2\times 10^{-5}$& $2\times 10^{-4}$& $2\times 10^{-4}$& $\mathbf{-1\times 10^{-4}}$\\
\hline
\end{tabular}}}
\subfloat[Spread difference between OPIM and \gc.]{
\scalebox{0.6}{
\label{tab:sd_opim}
\begin{tabular}{| c || c |c || c | c |c | }
\hline
\textbf{$b$}&  \textbf{YT-TV} & \textbf{YT-CO} &  \textbf{Stack-TV} & \textbf{Stack-CO} & {\bf Stack-LND}\\
\hline
\textbf{10}&  $\mathbf{-5\times 10^{-5}}$ & $\mathbf{-1\times 10^{-5}}$& $2\times 10^{-5}$& $\approx 0$& $1\times 10^{-5}$\\
\hline
\textbf{20}&  $\mathbf{-1\times 10^{-4}}$ & $1\times 10^{-5}$& $3\times 10^{-5}$& $2\times 10^{-5}$& $\mathbf{-2\times 10^{-5}}$\\
\hline
\textbf{50}&  $\mathbf{-2\times 10^{-4}}$ & $\mathbf{-3\times 10^{-5}}$ & $2\times 10^{-5}$ & $5\times 10^{-5}$& $\mathbf{-6\times 10^{-4}}$\\
\hline
\textbf{100}&  $\mathbf{-3\times 10^{-4}}$ & $2\times 10^{-5}$& $1\times 10^{-4}$& $7\times 10^{-5}$& $\mathbf{-2\times 10^{-4}}$\\
\hline
\textbf{150}& $\mathbf{-3\times 10^{-4}}$ & $\mathbf{-2\times 10^{-5}}$& $1\times 10^{-4}$& $1\times 10^{-4}$& $\mathbf{-3\times 10^{-4}}$\\
\hline
\textbf{200} & $\mathbf{-4\times 10^{-4}}$ & $\mathbf{-7\times 10^{-5}}$& $2\times 10^{-4}$& $2\times 10^{-4}$& $\mathbf{-3\times 10^{-4}}$\\
\hline
\end{tabular}}}
 \caption{Comparison with respect to (a) \textsc{IMM} and (b) \textsc{OPIM} on YT and Stack. A \emph{negative} value, highlighted in bold, indicates \emph{better performance by \gc.}}
 \end{table}

\begin{figure*}[t]
	\centering
\vspace{-0.20in}
\subfloat[Scalability]{
\label{fig:time_nips}
\includegraphics[width=1.2in]{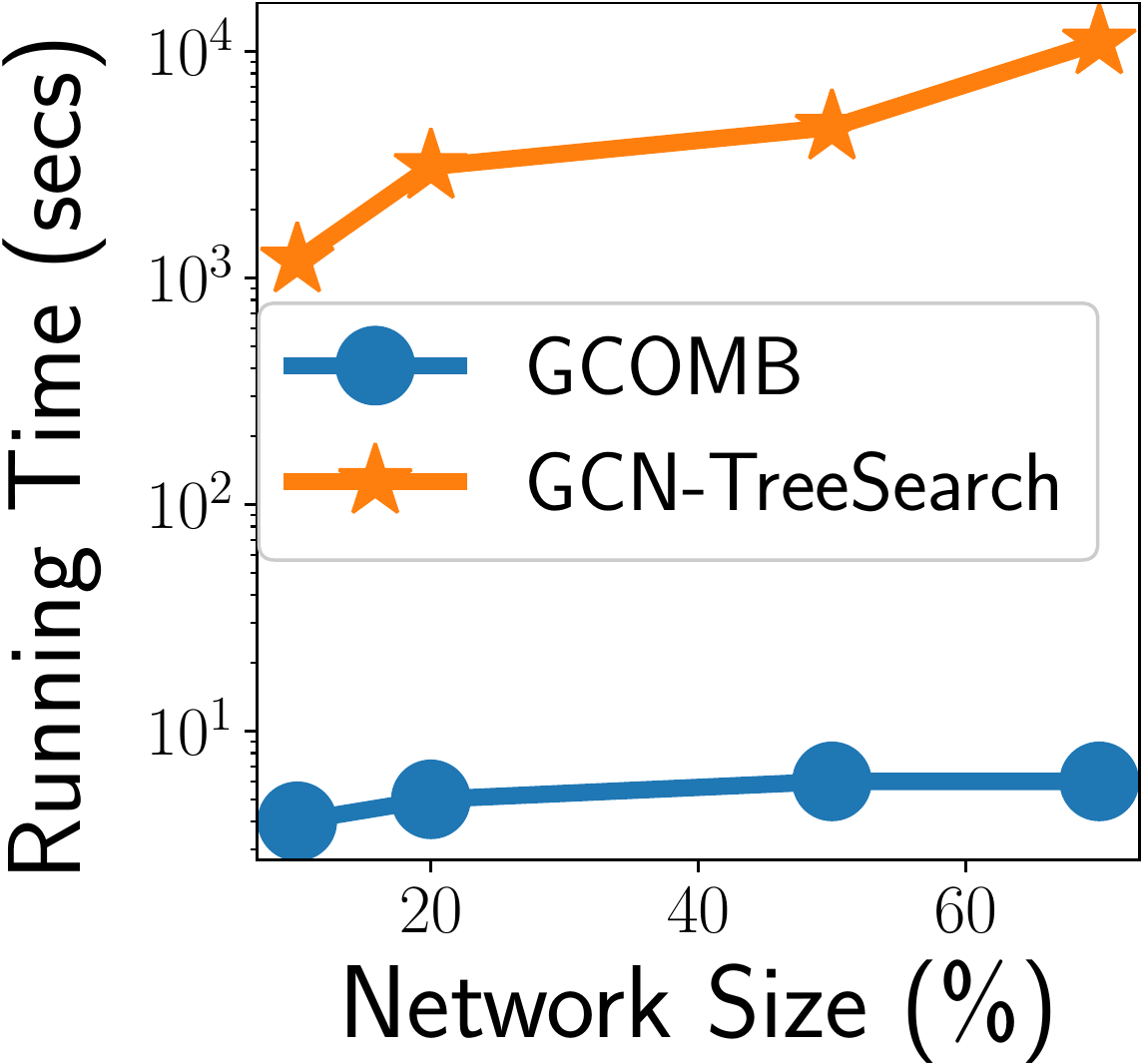}}\hspace{0.1in}
\subfloat[Speed-up IMM]{
	\label{fig:speedup_yt}
	\includegraphics[width=1.15in]{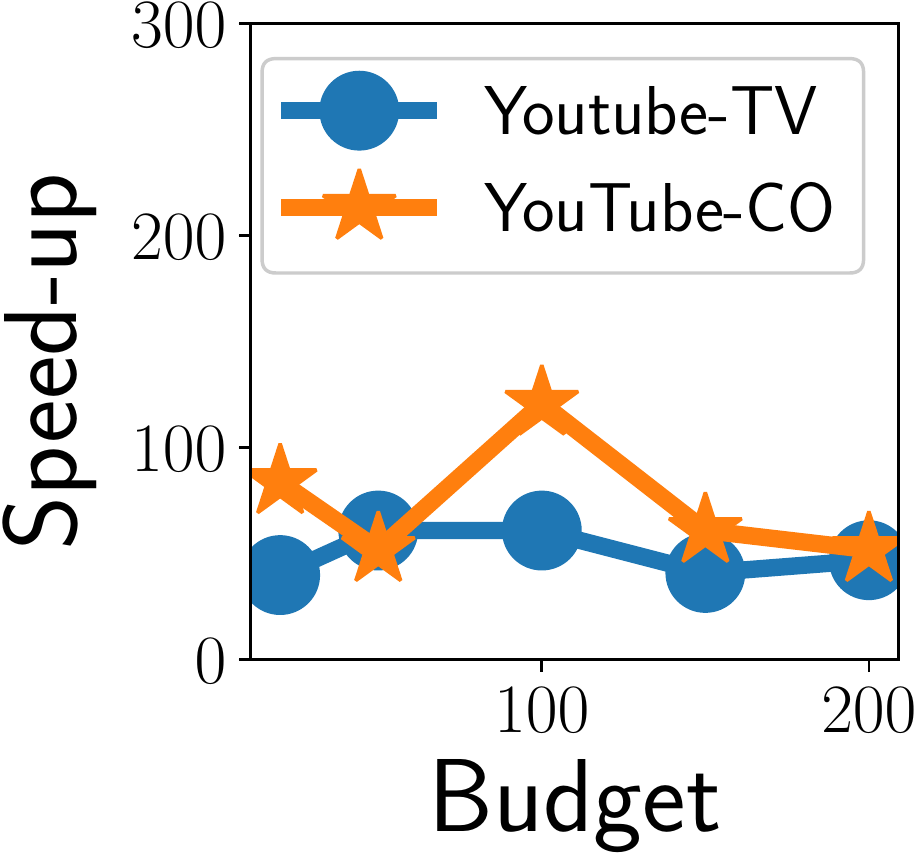}}\hspace{0.1in}
\subfloat[Speed-up IMM]{
	\label{fig:speedup_st}
	\includegraphics[width=1.12in,]{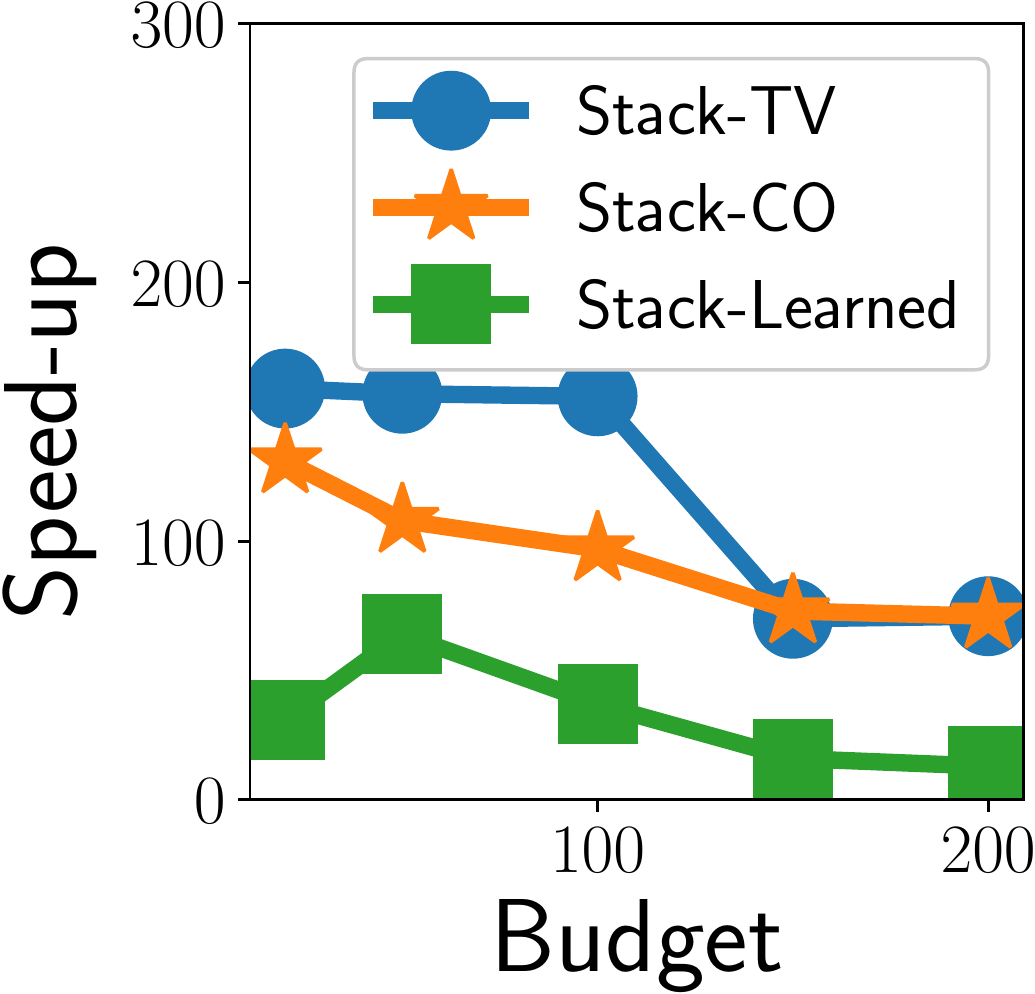}}\hspace{0.13in}
\subfloat[NoisePruner+CELF]{
	\label{fig:celf_noise}
	\includegraphics[width=1.17in,height=1.05in]{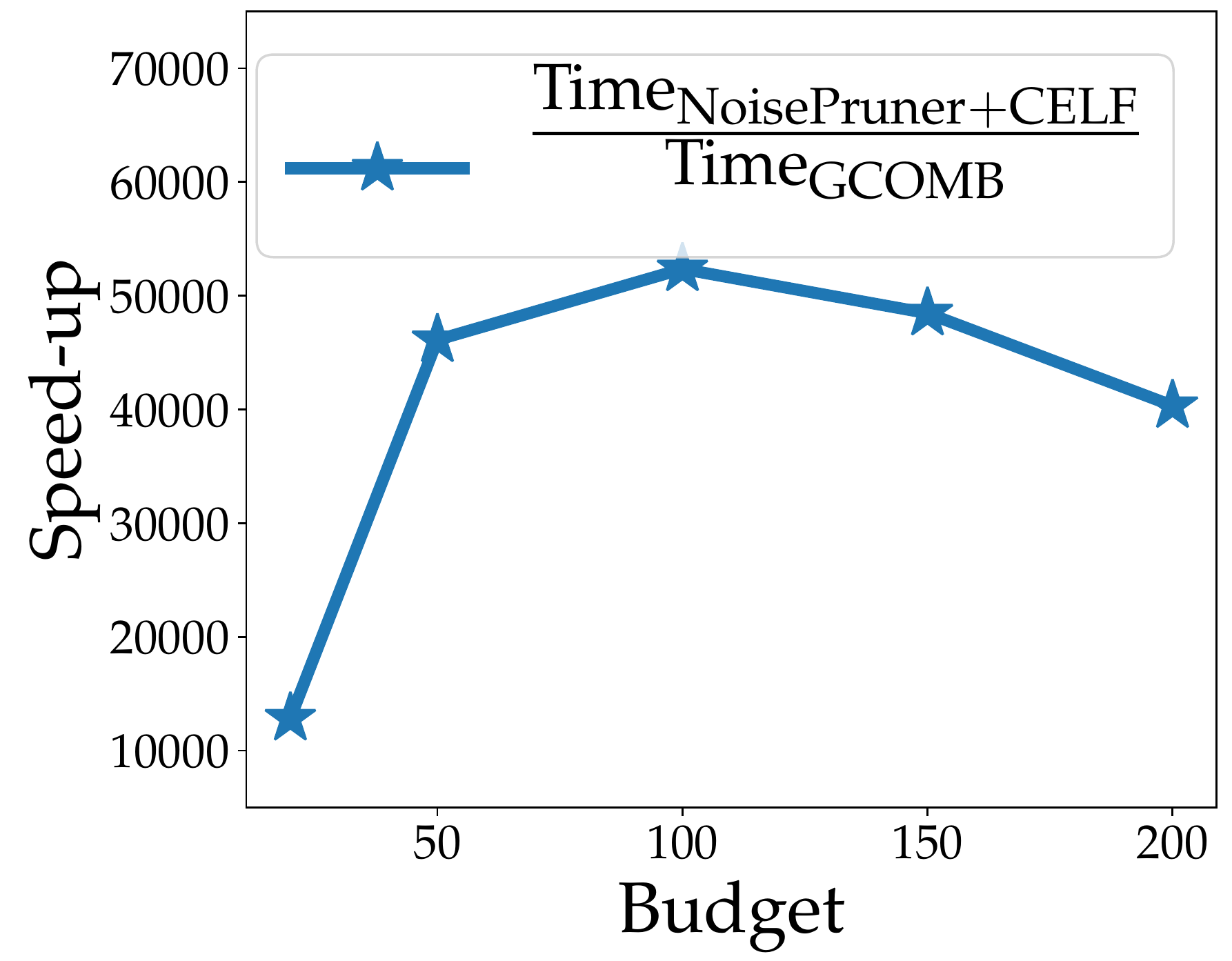}}\\
\hspace{0.16in}\subfloat[Speed-up OPIM]{
	\label{fig:speedup_yt_opim}
	\includegraphics[width=1.22in,height=1.05in]{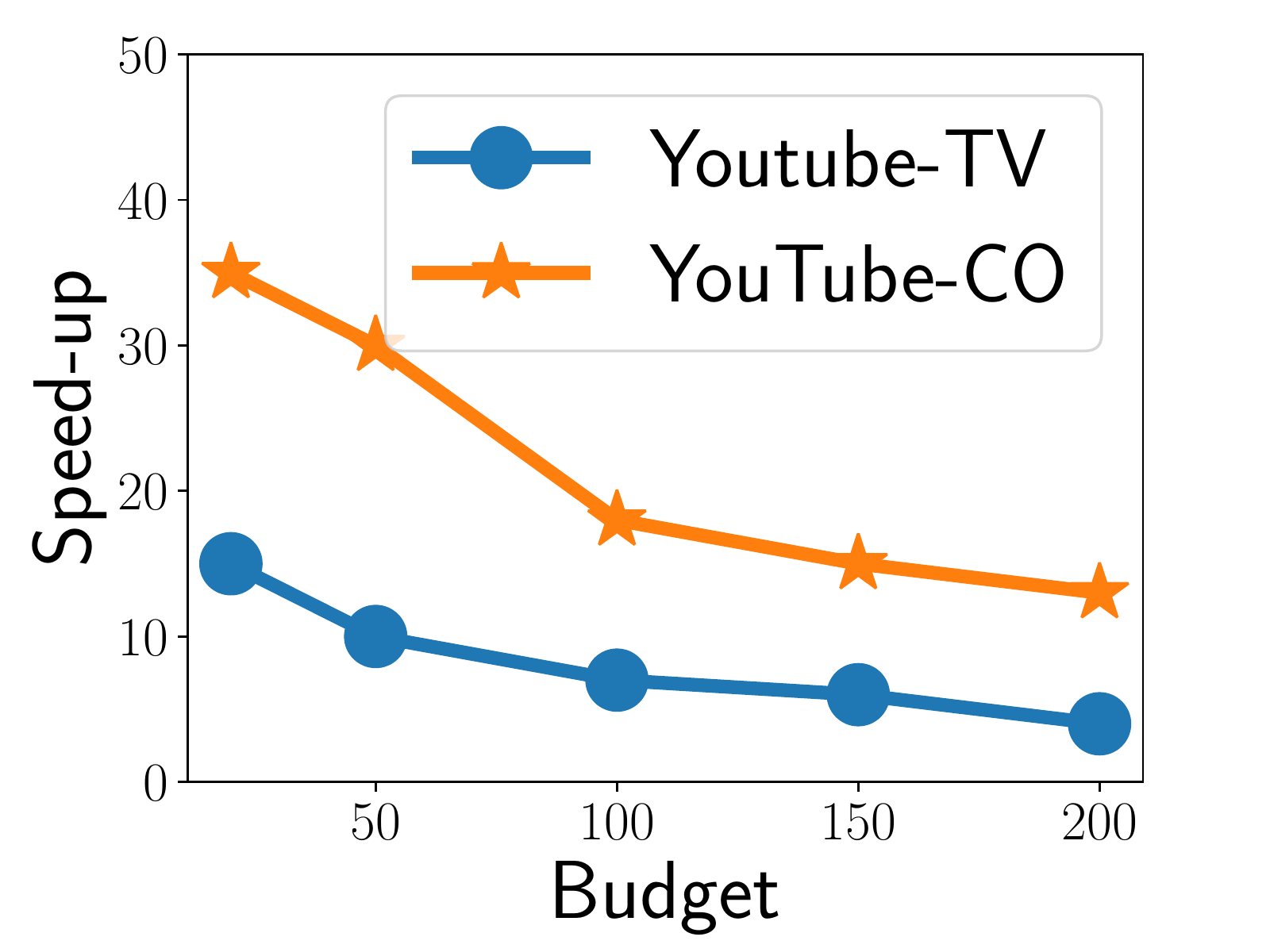}}\hspace{0.05in}
\subfloat[Speed-up OPIM]{
	\label{fig:speedup_st_opim}
	\includegraphics[width=1.23in,height=1.01in]{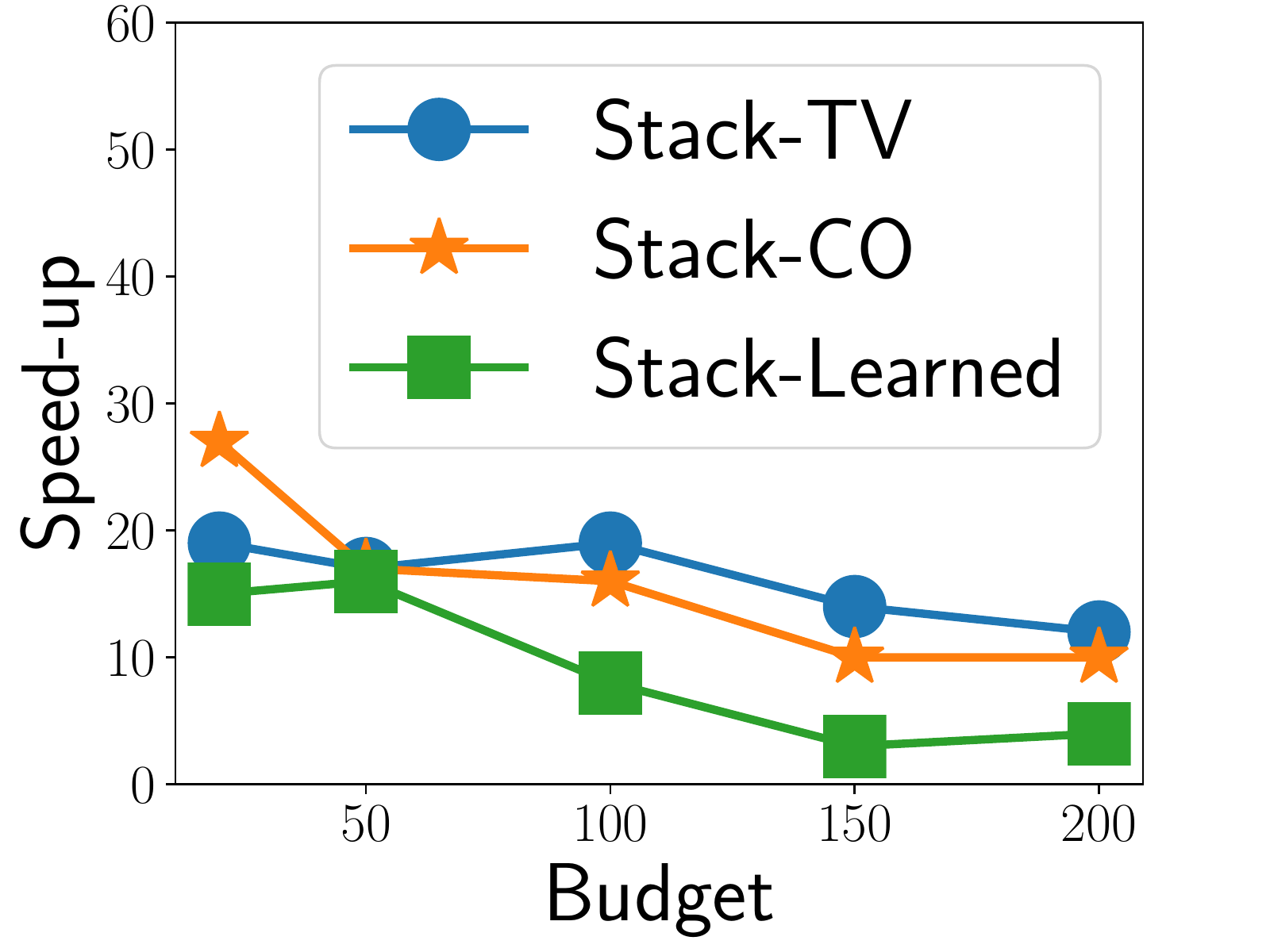}}\hspace{0.07in}
\subfloat[Running time (CO)]{
	\label{fig:time_ic}
	\includegraphics[width=1.12in]{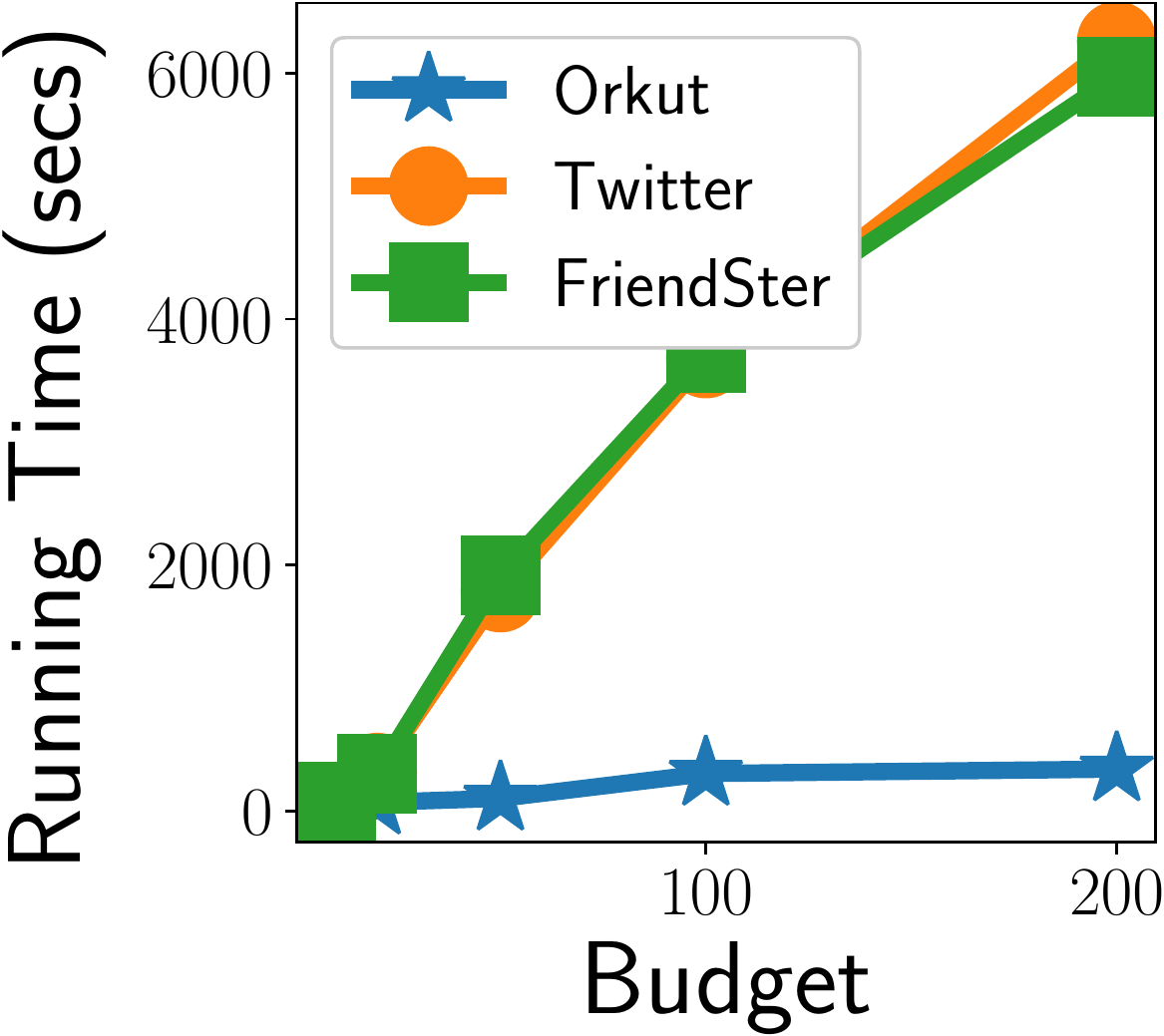}}\hspace{0.15in}
\subfloat[Running time (TV)]{
	\label{fig:time_tv}
	\includegraphics[width=1.15in]{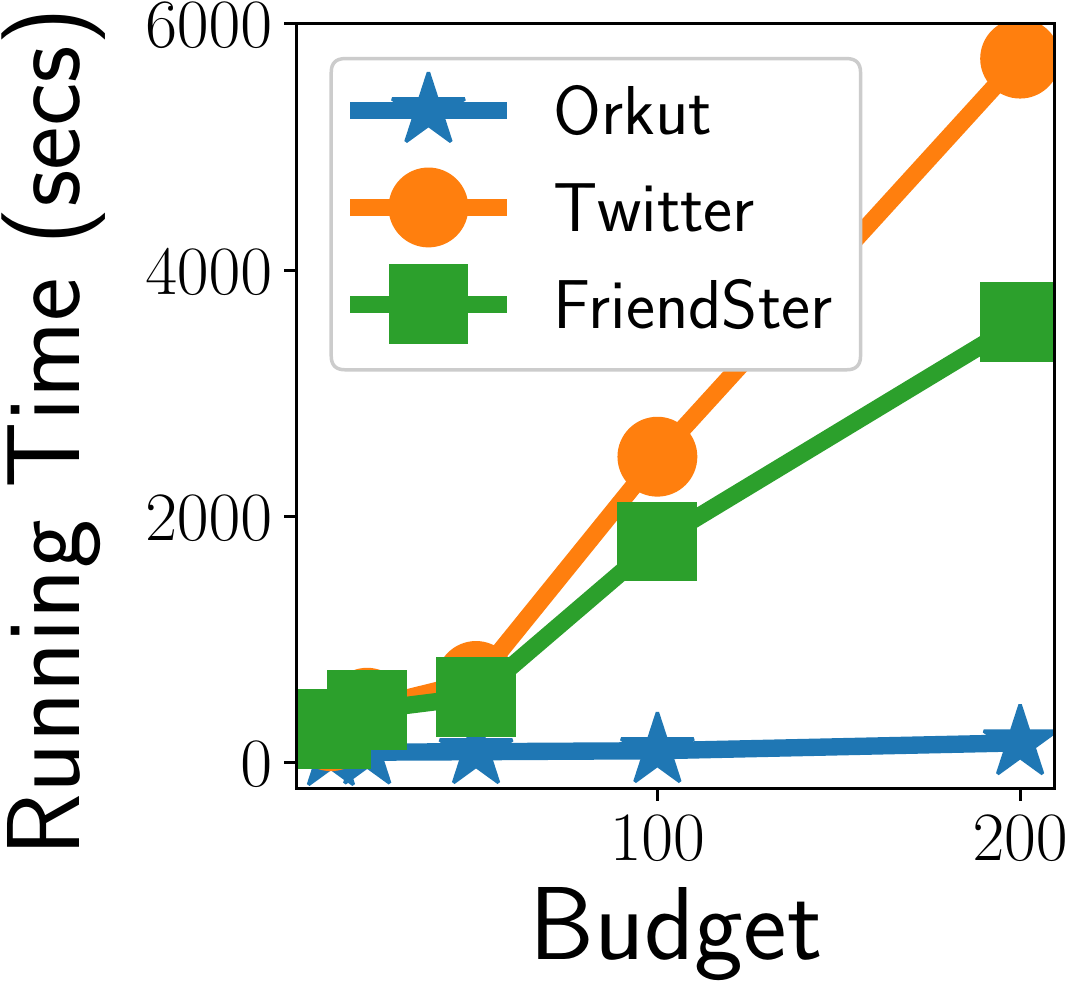}}

\subfloat[GCN Vs. QL]{
	\label{fig:gcnvsrl_time}
	\includegraphics[width=1.1in]{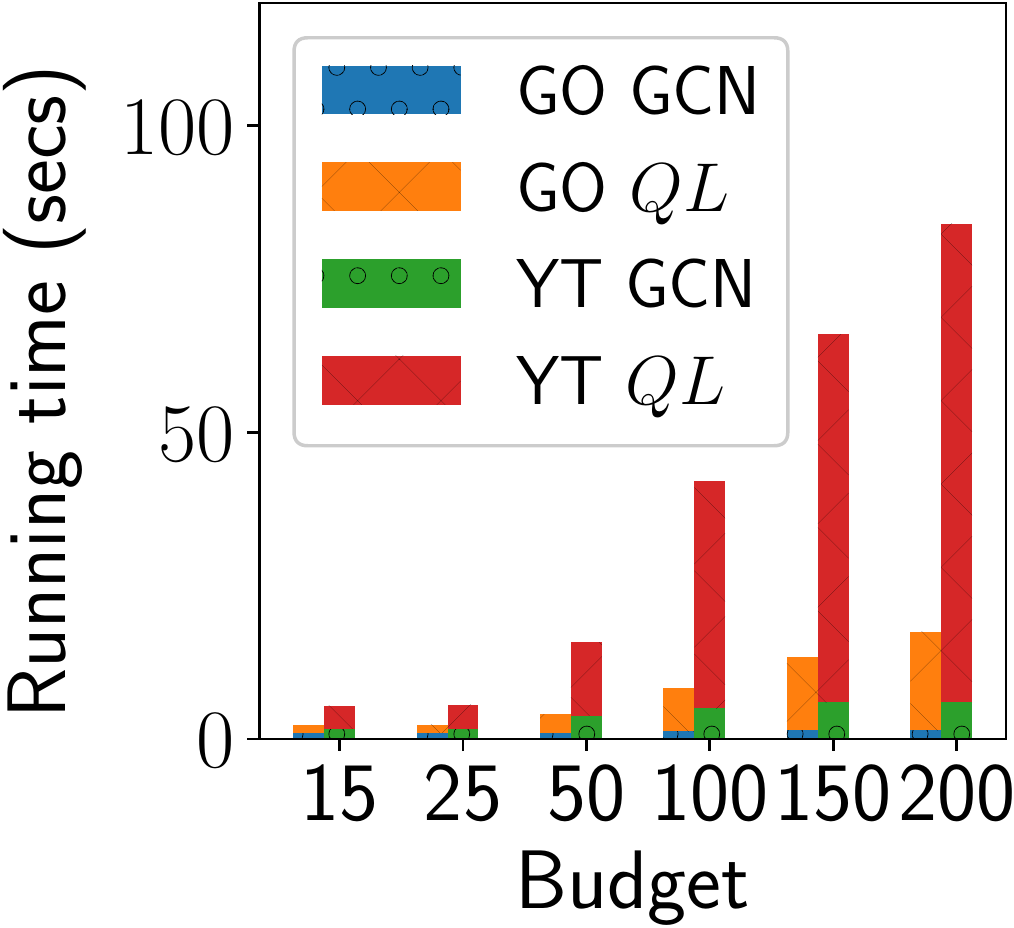}}\hspace{0.1in}
\subfloat[GCN Vs. QL]{
	\label{fig:gsvsrl}
	\includegraphics[width=1.14in]{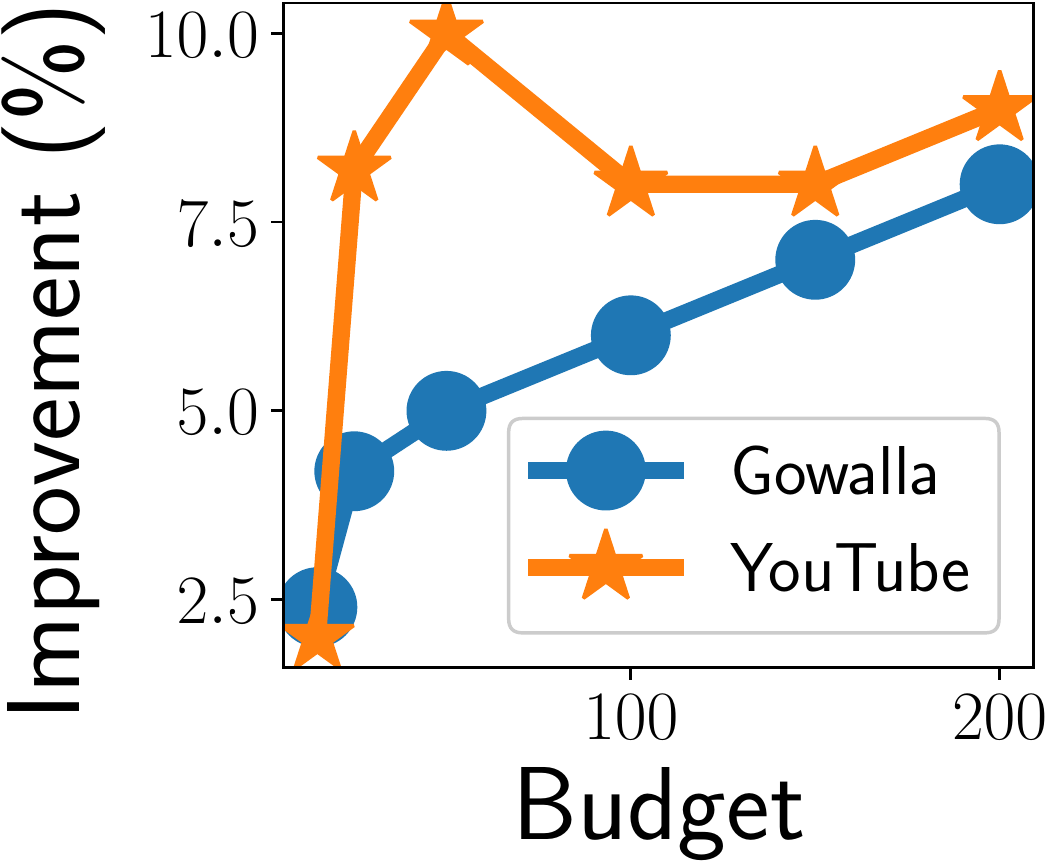}}\hspace{0.13in}
	\subfloat[MCP in Gowalla]{
	\label{fig:gcomb_v_vs_vg_run}
	\includegraphics[width=1.1in]{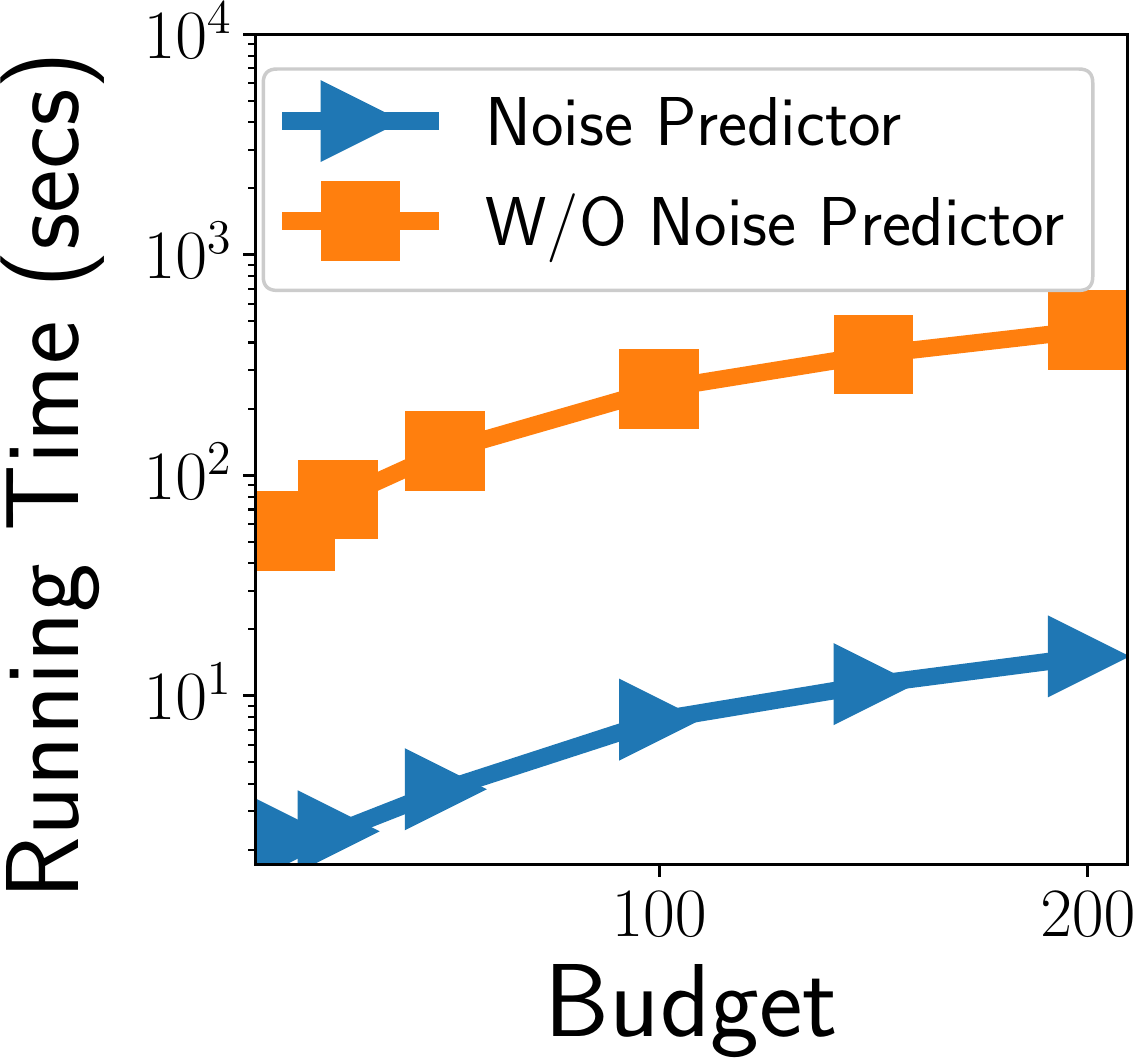}}\hspace{0.18in}
\subfloat[MCP in Gowalla]{
	\label{fig:gcomb_v_vs_vg_cov}
	\includegraphics[width=1.13in]{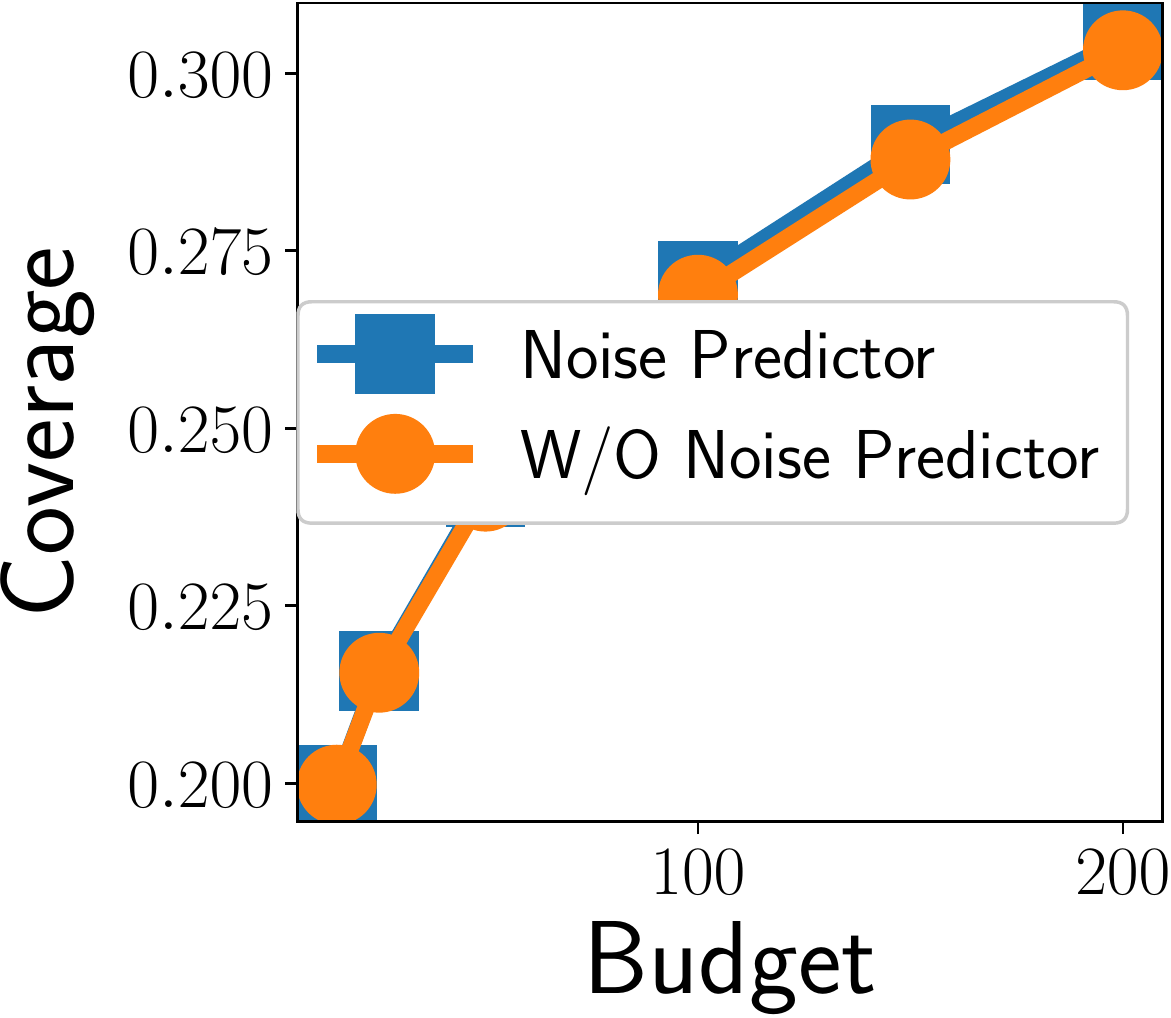}}
	\caption{ (a) Comparison of running time between \gc and \gt in YT  at $b=20$. (b-c) Speed-up achieved by \gc over IMM. (d) Speed-up achieved by \gc over NoisePruner+CELF on IM. (e-f) Speed-up achieved by \gc over OPIM. (g-h) Running times of \gc in IM in large graphs under the CO and TV edge models. (i) Distributions of running time between GCN and $Q$-learning in GO and YT datasets for MCP. (j) Improvement of $Q$-learning over GCN in MCP. (k-l) Impact of noise predictor on (k) running time and (l) quality.}
\end{figure*}

\subsection{Design Choices}
\textbf{Impact of $Q$-learning:} Since GCN predicts the expected marginal gain of a node, why not simply select the top-$b$ nodes with the highest predicted marginal gains for the given budget $b$? This is a pertinent question since, as visible in Fig.~\ref{fig:gcnvsrl_time}, majority of the time in \gc is spent on $Q$-learning. Fig.~\ref{fig:gsvsrl} shows that $Q$-learning imparts an additional coverage of up to $10\%$. 
 Improvement ($\%$) is quantified as $\frac{Coverage_{\gc}-Coverage_{GCN}}{Coverage_{GCN}}\times 100$.
	
\textbf{Impact of Noise Predictor: } Fig.~\ref{fig:gcomb_v_vs_vg_run} presents the impact of noise predictor which is close to two orders of magnitude reduction in running time. This improvement, however, does not come at the cost of efficacy (Fig.~\ref{fig:gcomb_v_vs_vg_cov}). In fact, the quality improves slightly due to the removal of noisy nodes. 
\vspace{-0.1in}

\section{Conclusion}

\sv \cite{dai2017learning} initiated the promising direction of learning combinatorial algorithms  on graphs. \gt \cite{li2018combinatorial} pursued the same line of work and enhanced  scalability to larger graphs. However, the barrier to million and billion-sized graphs remained. \gc removes this barrier with a new lightweight architecture. In particular, \gc uses a phase-wise mixture of supervised and reinforcement learning. While the supervised component predicts individual node qualities and prunes those that are unlikely to be part of the solution set, the $Q$-learning architecture carefully analyzes the remaining high-quality nodes to identify those that collectively form a good solution set.
This architecture allows \gc to generalize to unseen graphs of significantly larger sizes and convincingly outperform the state of the art in efficiency and efficacy. 
Nonetheless, there is scope for improvement. \gc is limited to set combinatorial problems on graphs. In future, we will explore a bigger class of combinatorial algorithms such as sequential and capacity constrained problems.


\pagebreak

\section*{Broader Impact}

The need to solve NP-hard combinatorial problems on graphs routinely arise in several real-world problems. Examples include facility location problems on road networks~\cite{medya2018noticeable}, strategies to combat rumor propagation in online social networks~\cite{budak}, computational sustainability~\cite{dilkina2011} and health-care~\cite{Wilder2018aamas}. Each of these problems plays an important role in our society. Consequently, designing effective and efficient solutions are important, and our current work is a step in that direction. The major impact of this paper is that good heuristics for NP-hard problems can be learned for large-scale data. While we are not the first to observe that heuristics for combinatorial algorithms can be learned, we are the first to make them scale to billion-size graphs, thereby bringing an algorithmic idea to practical use-cases. 




\begin{ack}
 The project was  partially supported by the National Science Foundation under award IIS-1817046. Further, Sahil Manchanda acknowledges the financial support from the Ministry of Human Resource Development (MHRD) of India and the Department of Computer Science and Engineering, IIT Delhi.
\end{ack}


\bibliographystyle{plain}
\bibliography{nips20}
\clearpage

\section{Appendix}
\renewcommand{\thesubsection}{\Alph{subsection}}
\subsection{Influence Maximization}
\label{app:im}
\begin{defn}[Social Network]
  A social network is denoted as an edge-weighted graph $G (V, E, W)$, where $V$ is the set of nodes (users), $E$ is the set of directed edges (relationships), and $W$ is the set of edge-weights corresponding to each edge in $E$. 
\end{defn}


The objective in \emph{influence maximization (IM)} is to maximize the \emph{spread} of influence in a network through activation of an initial set of $b$ \emph{seed} nodes. 
\begin{defn}[Seed Node]
  \label{def:seed}
  	A node $v \in V$ that acts as the source of information diffusion in the graph $G(V,E,W)$ is called a seed node. The set of seed nodes is denoted by $S$.
      \end{defn}
\vspace{-0.05in}
      \begin{defn}[Active Node]
	\label{def:active}
		A node $v \in V$ is deemed active if either (1) It is a seed node ($v \in S$) or (2) It is influenced by a previously active node $u \in V_{a}$. Once activated, the node $v$ is added to the set of active nodes $V_{a}$.
	      \end{defn}

Initially, the set of active nodes $V_{a}$ is the seed nodes $S$. 
The spread of influence is guided by the \emph{Independent Cascade (IC)} model. 
\vspace{-0.05in}
	      \begin{defn}[Independent Cascade \cite{kempe2003maximizing}]
		Under the IC model, time unfolds in discrete steps. At any time-step $i$, each newly activated node $u \in V_{a}$ gets one independent attempt to activate each of its outgoing neighbors $v$ with a probability $p_{(u,v)}=W(u,v)$. The spreading process terminates when in two consecutive time steps the set of active nodes remain unchanged. 
	      \end{defn}

	      \begin{defn}[Spread]
		\label{def:spread}
			%
			The spread $\Gamma(S)$ of a set of seed nodes $S$ is defined as the total proportion of nodes that are active at the end of the information diffusion process. Mathematically, $\Gamma(S)=\frac{|V_{a}|}{|V|}$.
		      \end{defn}

		      Since the information diffusion is a stochastic process, the measure of interest is the \emph{expected} value of spread. The \emph{expected} value of spread $f(\cdot)=\expectation[\flow(\cdot)]$ is computed by simulating the spread function a large number of times. The goal in IM, is therefore to solve the following problem.\\
		      
		      \textbf{Influence Maximization (IM) Problem \cite{kempe2003maximizing}:}
			Given a budget $b$, a social network $G$, and a information diffusion model $\mathcal{M}$, select a set $S^*$ of $b$ nodes such that the expected diffusion spread $f(S^*)=\expectation[\flow(S^*)]$ is maximized.

\subsection{The greedy approach}
\label{sec:greedy}
Greedy provides an $1-\frac{1}{e}$-approximation for all three NP-hard problems of MCP, MVC, and IM\cite{kempe2003maximizing}. Algorithm ~\ref{alg:greedy} presents the pseudocode. The input to the algorithm is a graph $G=(V,E)$, an optimization function $f(S)$ and the budget $b$. 
   Starting from an empty solution set $S$, Algorithm ~\ref{alg:greedy} iteratively builds the solution by adding the ``best'' node to $S$ in each iteration (lines 3-5). The best node $v^*\in V\backslash S$ is the one that provides the highest \emph{marginal gain} on the optimization function (line 4). The process ends after $b$ iterations.  

\paragraph{Limitations of greedy:} Greedy itself has scalability challenges depending on the nature of the problem. Specifically, in Alg.~\ref{alg:greedy} there are two expensive computations. First, computing the optimization function $f(\cdot)$ itself may be expensive. For instance, computing the expected spread in IM is $\#P$-hard~\cite{benchmarking}. Second, even if $f(\cdot)$ is efficiently computable, computing the marginal gain is often expensive. To elaborate, in MCP, computing the marginal gain involves a setminus operation on the neighborhood lists of all nodes $v\not\in S$ with the neighborhood of all nodes $u\in S$, where $S$ is the set of solution nodes till now. Each setminus operation consumes $O(d)$ time where $d$ is the average degree of nodes, resulting in a total complexity of $O(bd|V|)$. In IM, the cost is even higher with a complexity of $O(b|V|^2)$. In \gc, we overcome these scalability bottlenecks without compromising on the quality. \gc utilizes GCN \cite{hamilton2017inductive} to solve the first bottleneck of predicting $f(\cdot)$. Next, a deep $Q$-learning network is designed to estimate marginal gains efficiently. With this unique combination, \gc can scale to billion-sized graphs.		      
		      
		       \begin{algorithm}[t]
	\caption {The greedy approach}
	
	\label{alg:greedy}
	\begin{algorithmic}[1]
	 \REQUIRE $G = (V,E)$, optimization function $f(.)$, budget $b$
	 \ENSURE solution set $S$, $|S|=b$
		\STATE $S \leftarrow \emptyset$
		\STATE $i\leftarrow 0$
		\WHILE{$(i<b)$}
			\STATE $v^*\leftarrow \arg\max_{\forall v\in V\backslash S}\{f(S\cup\{v\})-f(S)\}$
			\STATE $S\leftarrow S\cup \{v^*\}$, $i\leftarrow i+1$
		\ENDWHILE
		\STATE \textbf{Return} $S$
   \end{algorithmic}
 \end{algorithm}
 
 		       \begin{algorithm}[t]
	\caption {The probabilistic greedy approach}
	\label{alg:pgreedy}
	\begin{algorithmic}[1]
	 \REQUIRE $G = (V,E)$, optimization function $f(.)$, convergence threshold $\Delta$
	 \ENSURE solution set $S$, $|S|=b$
		\STATE $S \leftarrow \emptyset$
		\WHILE{$(gain>\Delta)$}
			\STATE $v\leftarrow$ Choose with probability $\frac{f(S\cup\{v\})-f(S)}{\sum_{\forall v'\in V\backslash S}f(S\cup\{v'\})-f(S)}$
			\STATE $gain\leftarrow f(S\cup \{v\})-f(S)$
			\STATE $S\leftarrow S\cup \{v\}$
		\ENDWHILE
		\STATE \textbf{Return} $S$
   \end{algorithmic}
 \end{algorithm}
 
 
 \begin{algorithm}[t]
	\caption {Graph Convolutional Network (GCN)}
	\label{alg:gcn}
	\begin{algorithmic}[1]
	 \REQUIRE $G = (V,E)$, $\{score(v),$ input features $\mathbf{x}_v\:\forall v\in V\}$, budget $b$, noisy-node cut off $r^b_{max}$, depth $K$, weight matrices $\mathbb{W}^k,\:\forall k \in [1,k]$ and weight vector $\mathbf{w}$, dimension size $m_G$. 
	 \ENSURE Quality score $score'(v)$  for good  nodes and nodes in their 1-hop neighbors 
		\STATE $\mathbf{h}^0_v\leftarrow \mathbf{x}_v,\: \forall v\in V$
		\STATE $V^g \leftarrow  { \{ v \in V \mid rank(v,G) < r^b_{max} \}}$
		\STATE $V^{g,K} \leftarrow$ K-hop neighborhood of $V^g$
		
		\FOR{$k\in [1,K]$}
			\FOR{$v\in V^g \cup V^{g,K}$}
				\STATE $N(v)\leftarrow \{u|(v,u)\in E\}$
				\STATE $\mathbf{h}^k_N(v) \leftarrow \textsc{MeanPool} \left(\left\{h^{k-1}_u, \forall u\in N(v)\right\}\right)$ 
			        \STATE $\mathbf{h}^k_v \leftarrow ReLU \left (\mathbb{W}^k \cdot \textsc{Concat} \left (\mathbf{h}^k_N(v), h^{k-1}_v\right ) \right )$
			\ENDFOR
		\ENDFOR
		\STATE $\boldsymbol{\mu}_v\leftarrow \mathbf{h}_v^K,\:\forall v\in V^g\cup V^{g,1}$ 
		\STATE $score'(v) \leftarrow \mathbf{w}^T \cdot \boldsymbol{\mu }_v\:,\forall v\in V^g\cup V^{g,1} $
		
   \end{algorithmic}
 \end{algorithm}

\label{app:GCN_Training}
 \subsubsection{Training the GCN:}  For each node $v$, and its $score(v)$, which is generated using probabilistic greedy algorithm, we learn embeddings to predict this score via a Graph Convolutional Network (GCN)~\cite{hamilton2017inductive}. The pseudocode for this component is provided in Alg.~\ref{alg:gcn}.
 
 From the given set of training graphs $\{G_1,\cdots,G_t\}$, we sample a graph $G_i$ and a normalized budget $b$ from the range of budgets $(0,b^i_{max}]$, where $b^i_{max}=\max_{j=0}^m\left\{\frac{|S^i_j|}{|V_i|}\right\}$  . To recall, $S^i_j$ denotes the $j^{th}$ solution set constructed by probabilistic greedy on graph $G_i$.
 Further, quantity $r^b_{max}$ is computed from the set of training graphs and their probabilistic greedy solutions as described in  \S~\ref{sec:Training the GCN}. It is used to determine the nodes which are non-noisy for the budget $b$. 
 
 For a sampled training graph $G_i$ and budget $b$, only those nodes that have a realistic chance of being in the solution set are used to train the GCN (line 2). Each iteration in the outer loop represents the \emph{depth} (line 4). In the inner loop, we iterate over all nodes which are non-noisy and in their K-hop neighborhood (line 5). While iterating over node $v$, we fetch the current representations of $v$'s neighbors and \emph{aggregate} them through a \textsc{MeanPool} layer (lines 6-7). Specifically, for dimension $i$, we have:
$\mathbf{h}^k_N(v)_i=\frac{1}{|N(v)|}\sum_{\forall u\in N(v)} h^{k-1}_{u_i}$. 
 The aggregated vector is next \emph{concatenated} with the representation of $v$, which is then fed through a fully connected layer with \emph{ReLU} activation function (line 8), where ReLU is the \emph{rectified linear unit} ($ReLU(z) = max(0, z)$). 
The output of this layer becomes the input to the next iteration of the outer loop. Intuitively, in each iteration of the outer loop, nodes aggregate information from their local neighbors, and with more iterations, nodes incrementally receive information from neighbors of higher depth (i.e., distance).

At depth $0$, the embedding of each node $v$ is $h^0_v=\mathbf{x}_v$, while the final embedding is $\boldsymbol{\mu}_v= h^K_v$ (line 9). In hidden layers, Alg.~\ref{alg:gcn} requires the parameter set $\mathbb{W}=\{\mathbb{W}^k, k=1,2,\cdots,K\}$ to compute the node representations (line 8). Intuitively, $\mathbb{W}^k$ is used to propagate information across different depths of the model. To train the parameter set $\mathbb{W}$ and obtain predictive representations, the final representations are passed through another fully connected layer to obtain their predicted value $score'(v)$ (line 10). Further, the inclusion of 1-hop neighbors($V^{g,1}$) of $V^g$ in line 9 and line 10 is only for the importance sampling procedure. The parameters $\Theta_G$ for the proposed framework are therefore the weight matrices $\mathbb{W}$ and the weight vector $\mathbf{w}$. 
We draw multiple samples of graphs and budget and minimize the next equation using Adam optimizer \cite{adam} to learn the GCN parameters, $\Theta_G$.
\begin{equation}
J(\Theta_G)=\sum_{\sim\langle G_i, b \rangle}\frac{1}{|V^g_i|}\sum_{\forall v\in V^g_i}(score(v)-score'(v))^2
\end{equation}
In the above equation,  $V^g_i$ denotes the set of good nodes for budget $b$ in graph $G_i$.

\textbf{Defining $\mathbf{x}_v$: }  The initial feature vector $\mathbf{x}_v$ at depth $0$ should have the raw features that are relevant with respect to the combinatorial problem being solved. For example, in Influence Maximization (IM), the summation of the outgoing edge weights of a node is an indicator of its own spread.

\subsection{Q-learning}
\label{app:qlearning}
The pseudocode of the Q-learning component is provided in Algorithm ~\ref{algo:qlearning}. 

\textbf{Exploration vs. Exploitation: }In the initial phases of the training procedure, the prediction may be inaccurate as the model has not yet received enough training data to learn the parameters. Thus, with $\epsilon=\max\{0.05,0.9^t\}$ probability we select a random node from $C_t$. Otherwise, we trust the model and choose the predicted best node. Since $\epsilon$ decays exponentially with $t$, as more training samples are observed, the likelihood to trust the prediction goes up. This policy is commonly used in practice and inspired from bandit learning \cite{dai2017learning}.

\textbf{$n$-step $Q$-learning: }$n$-step $Q$-learning incorporates delayed rewards, where the final reward of interest is received later in the future during an episode (lines 6-9 in Alg.~\ref{algo:qlearning}). The key idea here is to wait for $n$ steps before the approximator’s parameters are updated and therefore, more accurately estimate future rewards.

\textbf{Fitted $Q$-learning: }For efficient learning of the parameters, we perform \textit{fitted $Q$-iteration} \cite{riedmiller2005neural}, which results in faster convergence using a neural network as a function approximator \cite{mnih2013playing}. Specifically, instead of updating the $Q$-function sample-by-sample, fitted $Q$-iteration uses \textit{experience replay} with a batch of samples. Note that the training process in Alg.~\ref{algo:qlearning} is independent of budget. The $Q$-learning component learns the best action to take under a given circumstance (state space).
\begin{algorithm}[h]
\caption{Learning $Q$-function}
\label{algo:qlearning}
\begin{algorithmic}[1] 
\REQUIRE  $\forall v\in V^g,\:score'(v)$, hyper-parameters $M$, $N$ relayed to fitted $Q$-learning, number of episodes $L$ and sample size $T$.
\ENSURE Learn parameter set $\Theta_Q$
\STATE Initialize experience replay memory $M$ to capacity $N$
\FOR{episode $e \leftarrow 1$ to $L$}
\FOR{step $t \leftarrow 1$ to $T$}
\STATE $v_t \leftarrow
\begin{cases}
\text{random node }v \not\in S_t \text{ with probability }\epsilon=\max\{0.05,0.9^t\} \\
\text{argmax}_{v \not\in S_t} Q'_n(S_t,v,\Theta_Q)\text{ otherwise}
\end{cases}$
\STATE $S_{t+1}\leftarrow S_t\cup\{v_t\}$
\IF{ $t \geq n$}
\STATE Add tuple $(S_{t-n}, v_{t-n}, \sum_{i=t-n}^{t}r(S_i,v_i), S_t)$ to $M$ 
\STATE Sample random batch $B$ from $ M$
\STATE Update $\Theta_Q$ by Adam optimizer for $B$
\ENDIF
\ENDFOR
\ENDFOR
\RETURN $\Theta_Q$
\end{algorithmic}
\end{algorithm}


\subsection{Complexity Analysis of the Test Phase}
\label{app:complexity}
For this analysis, we assume the following terminologies. $d$ denotes the average degree of a node. $m_{G}$ and $m_Q$ denote the embedding dimensions in the GCN and $Q$-learning neural network respectively. As already introduced earlier, $b$ denotes the budget and $V$ is the set of all nodes.

\subsubsection{Time Complexity}

In the test phase, a forward pass through the GCN is performed. Although the GCN's loss function only minimizes the prediction with respect to the good nodes, due to message passing from neighbors, in a $K$-layered GCN , we need the $K$-hop neighbors of the good nodes (we will denote this set as $V^{g,K}$). 
Each node in $V^{g,K}$ draws messages from its neighbors on which first we perform \textsc{MeanPool} and then dot products are computed to embed in a $m_G$-dimensional space. Applying \textsc{MeanPool} consumes $O(dm_{G})$ time since we need to make a linear pass over $d$ vectors of $m_G$ dimensions. Next,  we perform $m_G$ dot-products on vectors of $m_G$ dimensions. Consequently, this consumes $O(m_G^2)$ time. Finally, this operation is repeated in each of the $K$ layers of the GCN. Since $K$ is typically $1$ or $2$, we ignore this factor. Thus, the total time complexity of a forward pass is $O(|V^{g,K}|(dm_G+m_G^2))$. 

The budget ($b$) number of forward passes are made in the $Q$-learning component over only $V^g$ (the set of good non-noisy nodes). In each pass, we compute locality and the predicted reward. To compute locality, we store the neighborhood as a hashmap, which consumes $O(d)$ time per node. Computing predicted reward involves dot products among vectors of $O(m_Q)$ dimensions. Thus, the total time complexity of the $Q$-learning component is $O(|V^g|b(d+m_Q))$.

For noise predictor, we need to identify the top-$l$ nodes based on $\mathbf{x_v}$ (typically the out-degree weight). $l$ is determined by the noise predictor as a function of $b$. This consumes $|V|log(l)$ time through the use of a min-Heap.

Combining all three components, the total time complexity of  \gc is $O(|V|log(l)+|V^{g,K}|(dm_G+m_G^2)+|V^g|b(d+m_Q))$. Typically, $l<<|V|$ (See Fig.~\ref{fig:vb}) and may be ignored. Thus, the total time complexity is $\approx O(|V|+|V^{g,K}|(dm_G+m_G^2)+|V^g|b(d+m_Q))$.

\subsubsection{Space Complexity}

During testing, the entire graph is loaded in memory and is represented in linked list form which takes $O(|V|+|E|)$ space. The memory required for K layer GCN is $O(Km_G^2)$. Overall space complexity for GCN phase is $O(|V|+|E| + Km_G^2)$.

For the Q-learning component, entire graph is required for importance sampling purpose. It requires $O(|V|+|E|)$ space. Further, the space required for parameters for Q-network is $O(m_Q)$, since input dimension for Q-network is fixed to 2. Thus, space complexity of Q-network is $O(|V| + |E| + m_Q)$.
Therefore, total space complexity of \gc is $O(|V|+|E| + Km_G^2 + m_Q)$.

\subsection{Number of parameters}
\label{app:neuralparameters}
\textbf{GCN:} If $m_G$ is the embedding dimension, each $\mathbb{W}_k$ is a matrix of dimension $m_G^2$. Other than $\mathbb{W}_k$, we learn another parameter $\boldsymbol{w}$ in the final layer (line 10 of Alg.~\ref{alg:gcn}) of $m_G$ dimension. Thus, the total parameter size is $K\times m_G^2+m_G$, where $K$ is the number of layers in GCN.

\textbf{Q-learning:} If $m_Q$ is the dimension of the hidden layer in $Q$-learning, each of $\Theta_1$, $\Theta_2$, and $\Theta_3$ is a matrix of dimension $m_Q\times 2$. $\Theta_4$ is a vector of dimension $3m_Q$. Thus, the total number of parameters is $9m_Q$.
\subsection{Proof of Theorem~\ref{thm:samplesize}}
\label{app:samplesize}
 A sampling procedure is \emph{unbiased} if it is possible to estimate the mean of the target population from the sampled population, i.e., $\mathbb{E}[\hat{\mu}_{N_z(V^g)}]=\mu_{N(V^g)}=\frac{\sum_{v\in N(V^g)} I(v)}{|N(V^g)|}= \frac{1}{|N(V^g)|}$
, where $\hat{\mu}_{N_z(V^g)}$ is the \emph{weighted average} over the samples in $N_z(V^g)$. Specifically,  
\begin{equation}
    \hat{\mu}(N_z(V^g))= \frac{1}{\sum_{v\in N_z(V^g)}\hat{w}_v}\sum_{v\in N_z(V^g)} \hat{w}_v \cdot  I(v)
\end{equation}
 where $\hat{w}_v= \frac{1}{I(v)}$. 

\begin{lem} \label{lemma:imp_unbias}
Importance sampling is an unbiased estimate of $\mu_{N(V^g)}$, i.e., 
$\mathbb{E}\left[\hat{\mu}_{N_z(V^g)}\right]=\mu_{N(V^g)}$, if $\hat{w}_v= \frac{1}{I(v)}$.
\end{lem}
\begin{proof}
\begin{equation}
\nonumber
\mathbb{E}[\hat{\mu}_{N_z(V^g)}]= \frac{1}{\mathbb{E}[\sum_{v\in N_z(V^g)}\hat{w}_v]}\cdot\mathbb{E}\left[\sum_{v\in N_z(V^g)} \hat{w}_v \cdot  I(v)\right]
\end{equation}

If we simplify the first term, we obtain
\begin{alignat}{5}
\nonumber
\mathbb{E}\left[{\sum_{v\in N_z(V^g)}\hat{w}_v}\right] &= z\times\mathbb{E}[{\hat{w}_v}]\\
\nonumber
=z\times\sum_{\forall v\in N(V^g)}{{\hat{w}_v}\cdot  I(v)} 
&=|N(V^g)|\times z
\nonumber
\end{alignat}

From the second term, we get,
\begin{equation}
\nonumber
\mathbb{E}\left[\sum_{v\in N_z(V^g)} \hat{w}_v \cdot  I(v)\right]=z\times\mathbb{E}[\hat{w}_v \cdot  I(v)]=z
\end{equation}

Combining these two, $ \mathbb{E}[\hat{\mu}_{N_z(V^g)}]=
\frac{z}{|N(V^g)|\times z}=\mu_{N(V^g)}$.
\end{proof}

Armed with an unbiased estimator, we show that a bounded number of samples provide an accurate estimation of the locality of a node.
\begin{lem}\label{lemma:approx_sample}
[\textbf{Theorem 1 in main draft}] Given $\epsilon$ as the error bound,
$P\left[|\hat{\mu}_{N_z(V^g)} - \mu_{N(V^g)}| < \epsilon\right]>1-\frac{1}{|N(V^g)|^2}$,
where $z$ is $O\left(\frac{\log |N(V^g)|}{\epsilon^2}\right)$.
\end{lem}

\begin{proof}
The samples can be viewed as random variables associated with the selection of a node. More specifically, the random variable, $X_i$, is the importance associated with the selection of the $i$-th node in the importance sample $N_z(V^g)$. Since the samples provide an unbiased estimate (Lemma \ref{lemma:imp_unbias}) and are i.i.d., we can apply \emph{Hoeffding's inequality}~\cite{hoeff1963} to bound the error of the mean estimates:
\begin{equation}
\nonumber
P\left[|\hat{\mu}_{N_z(V^g)}-\mu_{N(V^g)}| \geq \epsilon\right ]\leq \delta
\end{equation}
where $\delta=2\exp\left(-\frac{2z^2\epsilon^2}{\mathcal{T}}\right)$, $\mathcal{T} =\sum\limits_{i=1}^{z}(b_i-a_i)^2$, and each $X_i$ is strictly bounded by the intervals $[a_i, b_i]$. Since we know that importance is bounded within $[0,1]$,  $[a_i,b_i]=[0,1]$. Thus, 
\begin{equation}
\nonumber
\delta=2\exp\left(-\frac{2z^2\epsilon^2}{z}\right) = 2\exp\left(-2z\epsilon^2\right)
\end{equation} 
By setting the number of samples $z=\frac{\log (2|N(V^g)|^2)}{2\epsilon^2}$,
we have, 
\begin{equation}
\nonumber
p\left[|\hat{\mu}_{N_z(V^g)} - \mu_{N(V^g)}| < \epsilon\right]>1-\frac{1}{|N(V^g)|^2}
\end{equation}
\end{proof}

\subsection{Training time distribution of different phases of \gc} 
\label{app:training_time_dist}

Fig. \ref{fig:prep_gc_rl_traintime} shows the distribution of time spent in different phases of training of \gc. Prob-Greedy refers to the phase in which probabilistic greedy algorithm is run on training graphs to obtain training labels for GCN component. Train-GCN and Train-QL refers to the training phases of GCN and Q-network respectively.

\begin{figure}[t]
\centering
	\includegraphics[width=1.6in]{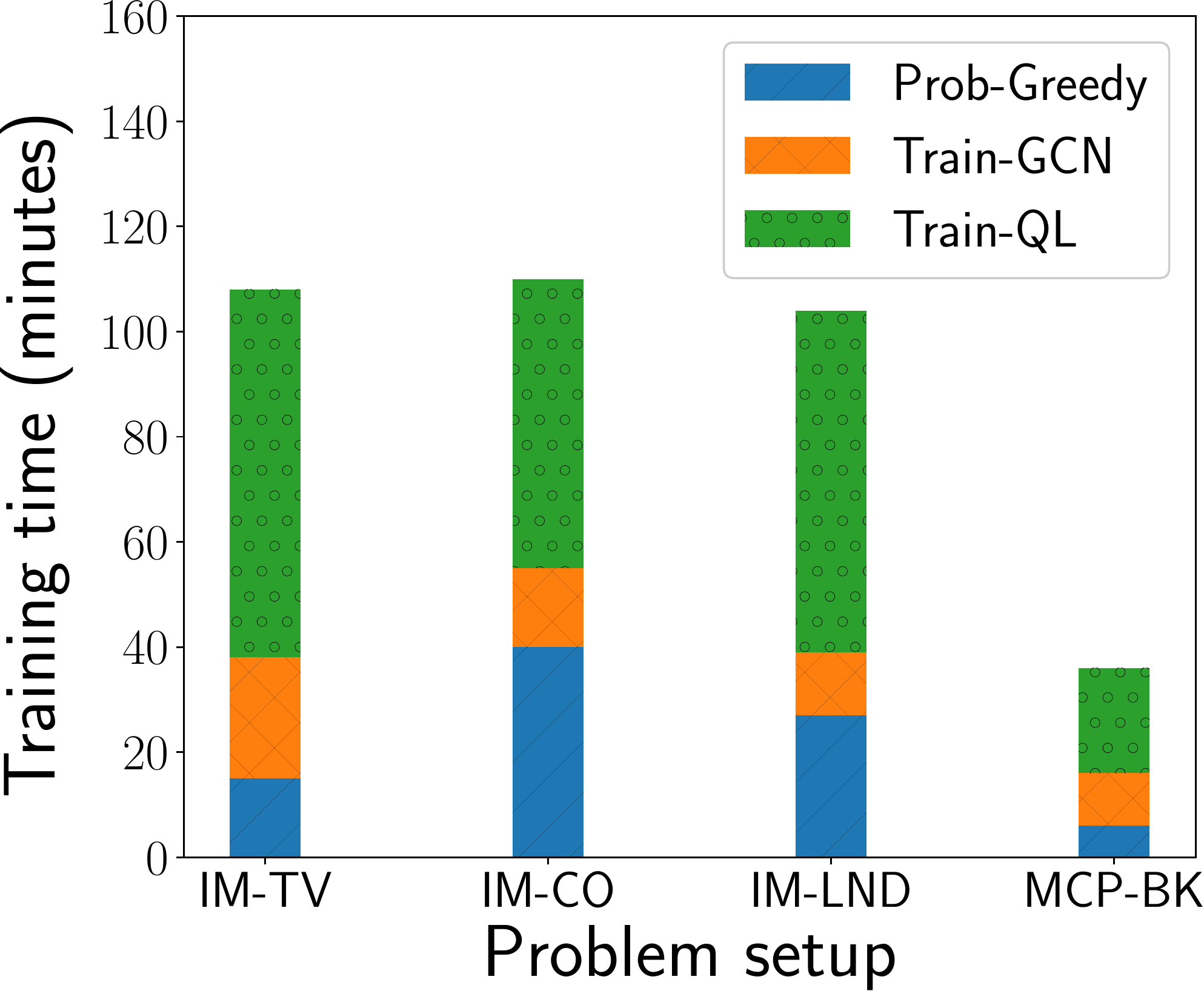}
	\caption{Phase-wise training time distribution of \gc}\label{fig:prep_gc_rl_traintime}
\end{figure}


\subsection{Parameters} 
\label{app:parameters}
\gc has two components: GCN and the $Q$-Learning part. GCN is trained for $1000$ epochs with a learning rate of $0.001$, a dropout rate of $0.1$ and a convolution depth ($K$) of $2$. The embedding dimension is set to $60$. For training the $n$-step $Q$-Learning neural network, $n$ and discount factor $\gamma$ are set to $2$ and $0.8$ respectively, and a learning rate of $0.0005$ is used. The raw feature $x_v$ of node $v$ in the first layer of GCN  is set to the summation of its outgoing edge weights. For undirected, unweighted graphs, this reduces to the degree. In each epoch of training, $8$ training examples are sampled uniformly from the Replay Memory with capacity $N=50$ as described in Alg.~\ref{algo:qlearning}. The sampling size $z$, in terms of percentage, is varied at $[1\%, 10\%,30\%,50\%,75\%,99\%]$ on the validation sets, and the best performing value is used. As we will show later in Fig.\ref{fig:sampling}, $10\%$ is often enough.

For all train sets, we split into two equal halves, where the first half is used for training and the second half is used for validation. For the cases where we have only one training graph, like  BrightKite(BK) in MCP and MVC, we randomly pick 50\% of the edges for the training graph and the remaining 50\% for the validation graph. 
The noise predictor interpolators in MCP are fitted on 10\% randomly edge-sampled subgraphs from Gowallah, Twitter-ew and YouTube. During testing, the remaining 90\%  subgraph is used, which is edge disjoint to the 10\% of the earlier sampled subgraph. 

\subsection{Extracting Subgraph from Gowalla}
\label{app:go}
To extract the subgraph, we select a node proportional to its degree. Next, we initiate a breadth-first-search from this node, which expands iteratively till $X$ nodes are reached, where $X$ is the target size of the subgraph to be extracted. All of these $X$ nodes and any edge among these nodes become part of the subgraph.

\subsection{Comparison with \textsc{Opim} on billion sized graphs }
\label{app:IM_OPIM_Large}

\begin{figure}[t]
\centering
\subfloat[Speed-up OPIM Orkut]{
	\label{fig:opim_speed_up_orkut}
	\includegraphics[width=1.5in,  height=1.2in]{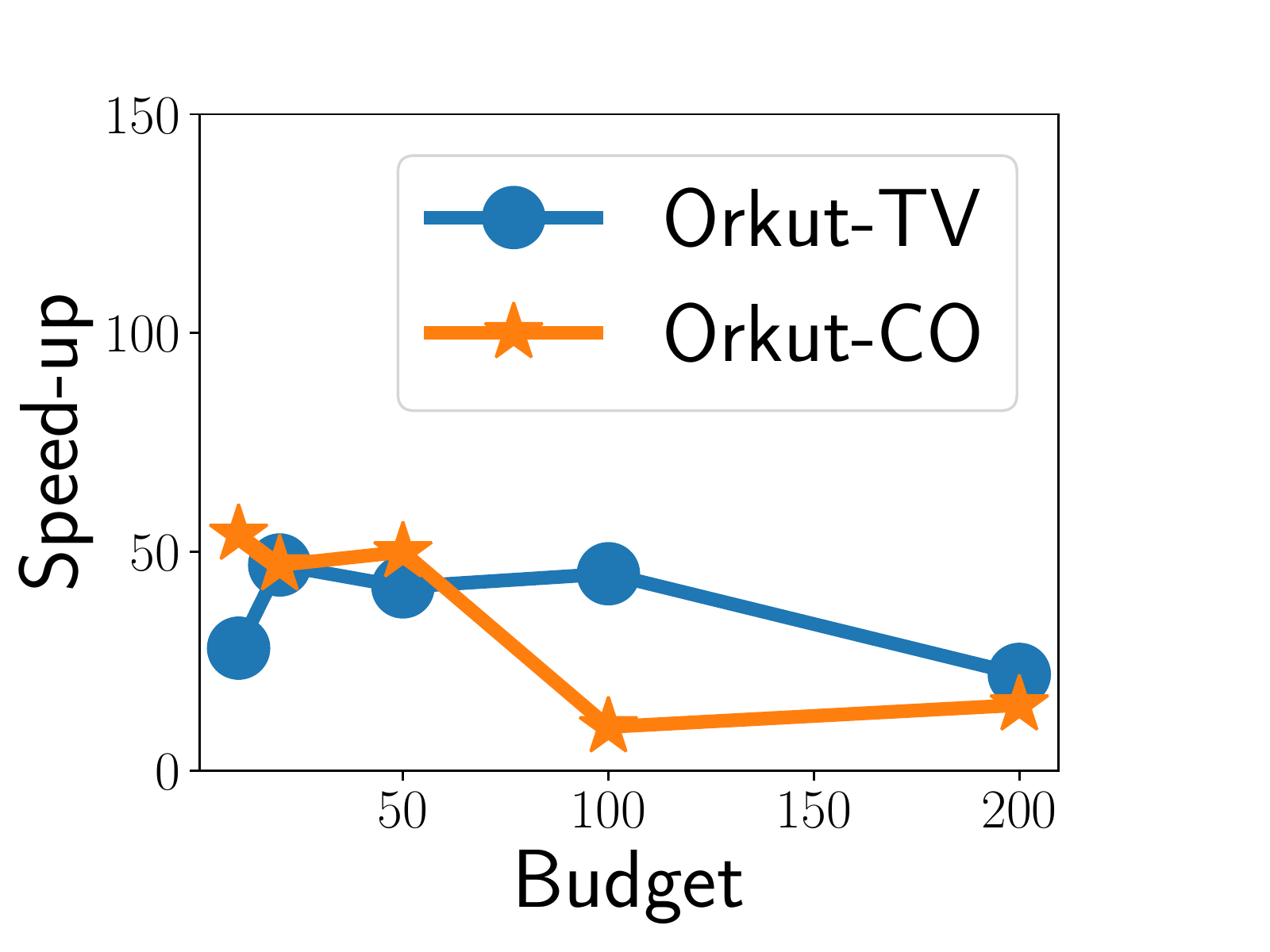}}
\subfloat[Speed-up OPIM Friendster]{
	\label{fig:opim_speed_up_friendster0.}
	\includegraphics[width=1.6in, height=1.2in]{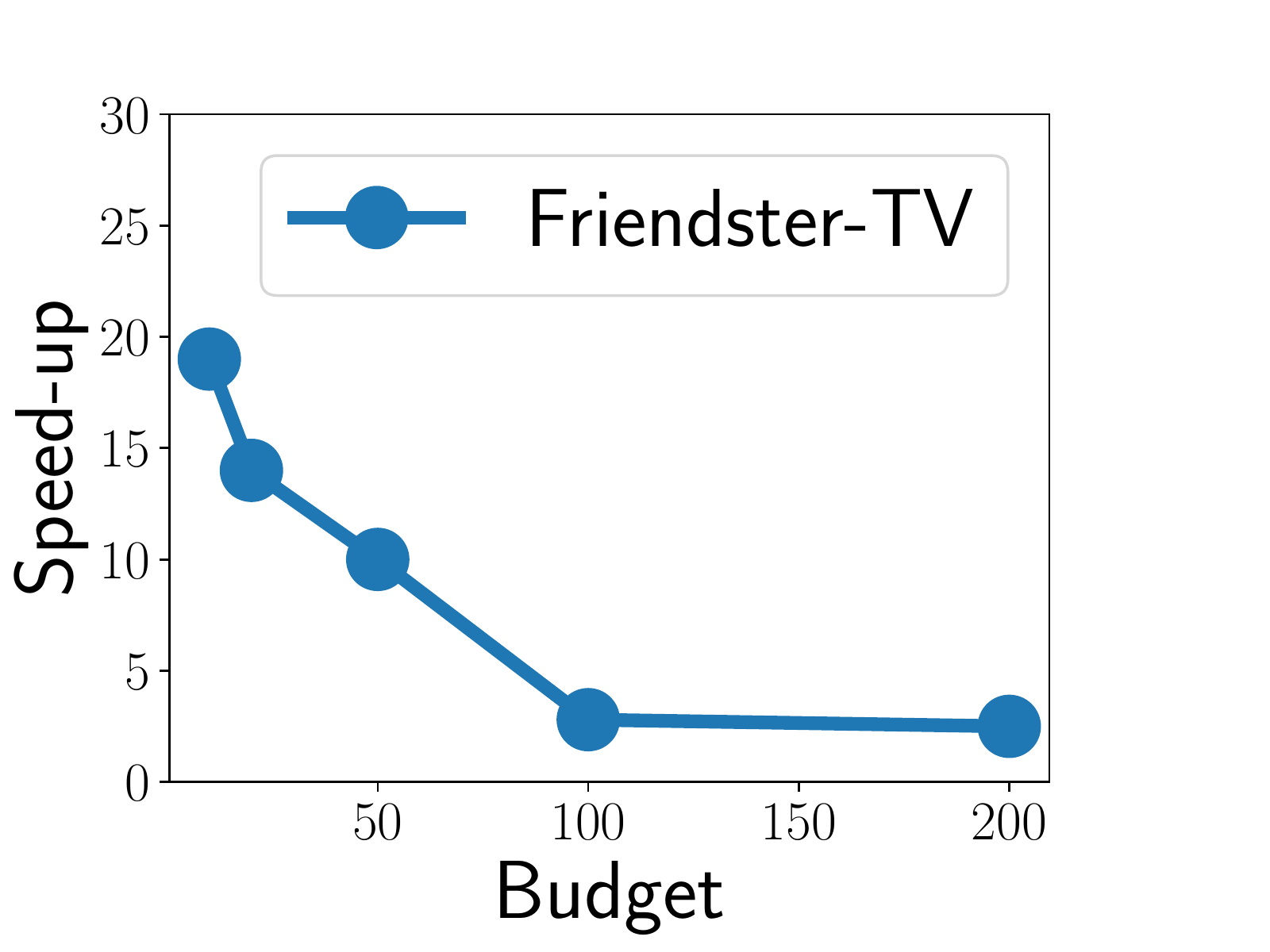}}
	
	\caption{Speed up obtained by \gc over \textsc{OPIM} }
\end{figure}

Figs.~\ref{fig:opim_speed_up_orkut}- ~\ref{fig:opim_speed_up_friendster0.} present the speed-up achieved by \gc over \textsc{OPIM} on Orkut and Friendster. Speed-up is measured as $\frac{time_{OPIM}}{time_{\gc}}$ where $time_{OPIM}$ and $time_{\gc}$ are the running times of OPIM and \gc respectively. OPIM crashes on Friendster-CO and Twitter dataset.
\subsection{Results on Max Cover Problem (MCP) on Gowalla }
\label{app:MCP_gowalla}

\begin{figure}[t]
\centering
\subfloat[Quality]{
	\label{fig:gowalla_quality_mcp}
	\includegraphics[width=1.4in]{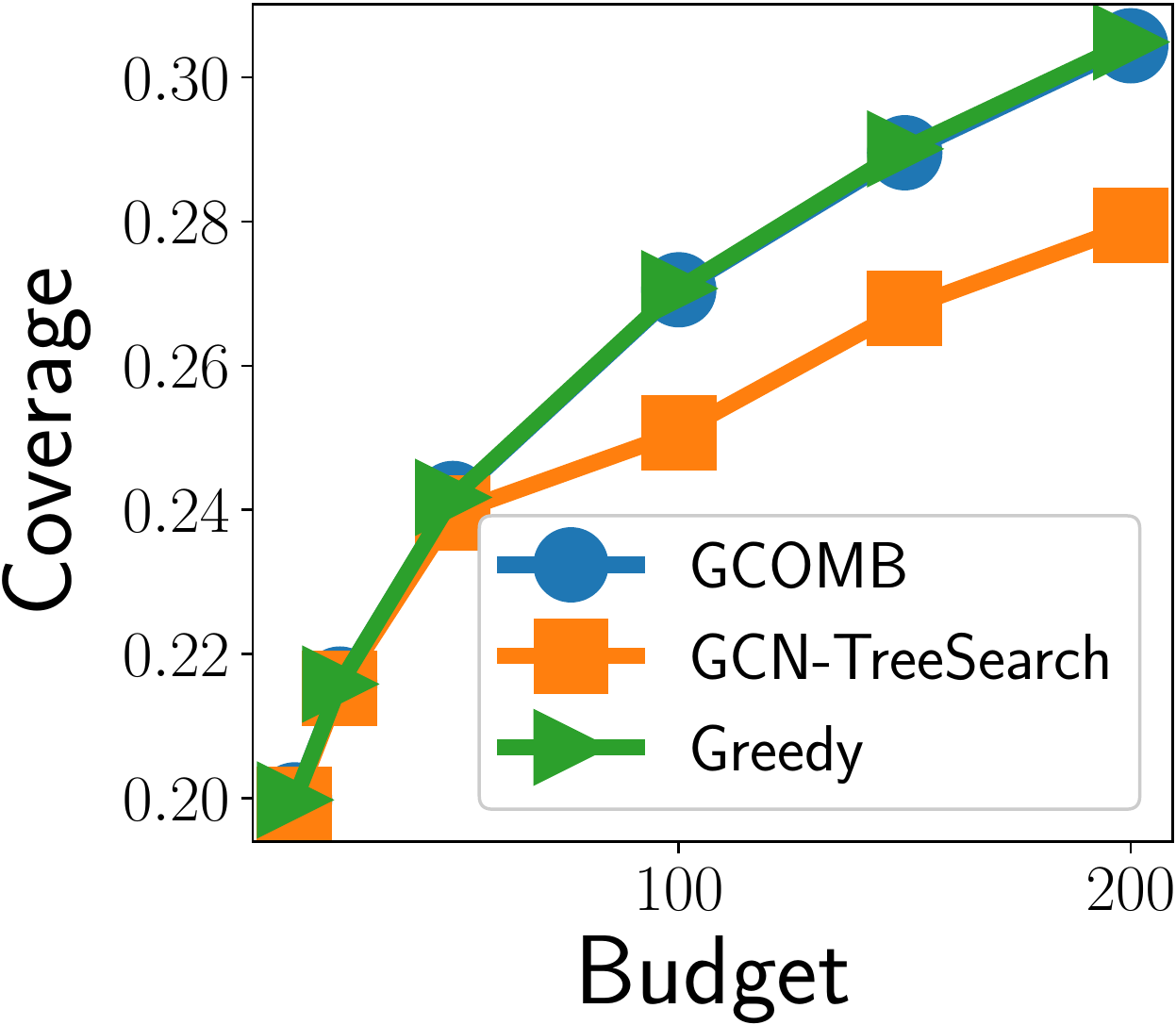}}
\subfloat[Scalability]{
	\label{fig:gowalla_running_time_mcp}
	\includegraphics[width=1.4in]{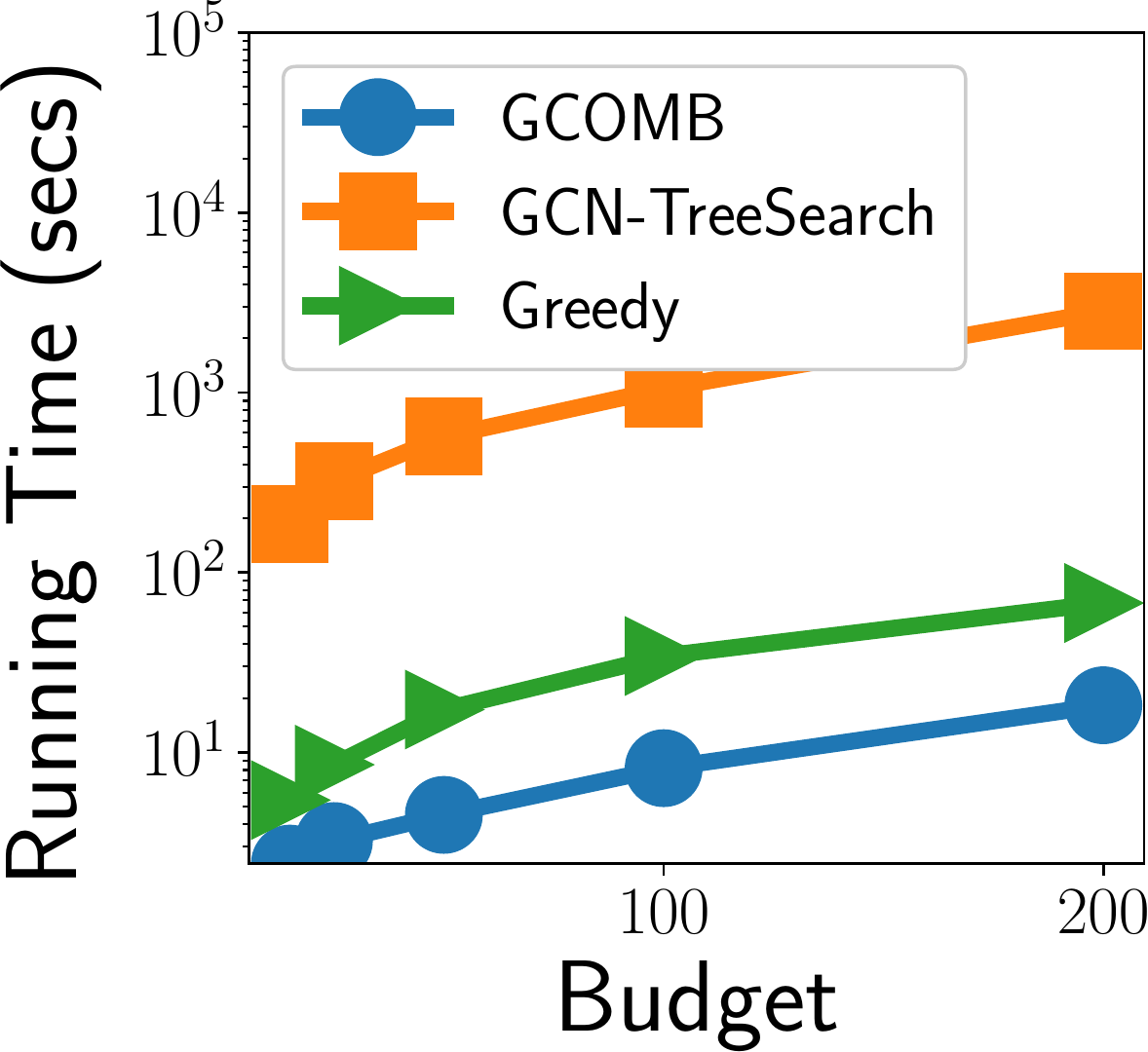}}
	
	\caption{MCP : Gowalla: a) Quality comparison of \gc and \gt against greedy. b) Running times of \gc and \gt against the greedy approach.}
\end{figure}

Fig.~\ref{fig:gowalla_quality_mcp} presents the impact of budget on Coverage on Gowalla dataset. The quality achieved by \gc is similar to Greedy, while \gt is inferior. \gc  is up to two orders of magnitude faster than \gt and $10$ times faster than Greedy as can be seen in Fig. \ref{fig:gowalla_running_time_mcp}.

\subsection{Comparison with Stochastic Greedy (SG) on MCP  }
\label{app:SG_MCP}

We compare the performance of \gc with \emph{SG} on MCP. As can be seen in Table ~\ref{table:SG_MCP}, \gc is up to $20\%$ better in quality at $\epsilon=0.2$ and yet $2$ times faster. SG fails to match quality even at $\epsilon=0.05$, where it is even slower. Furthermore, SG is not drastically faster than \textsc{Celf} in MCP due to two reasons: (1) cost of computing marginal gain is $O(Avg. degree)$, which is fast. (2) The additional data structure maintenance in SG to access sampled nodes in sorted order does not substantially offset the savings in reduced marginal gain computations.

\begin{table}[t]
\centering
\scalebox{0.7}{
\begin{tabular}{|c|p{0.6in}|p{0.9in}|p{0.9in}|}
\hline
\textbf{Budget} &\textbf{Speed-up $\epsilon=0.2$} & \textbf{Coverage Difference $\epsilon=0.2$} & \textbf{Coverage Difference $\epsilon=0.05$}\\
\hline
\textbf{20}& $2$ & $-0.09$ & $\mathbf{-0.001}$ \\
\hline
\textbf{50}& $2$ & $-0.13$ & $\mathbf{-0.003}$\\
\hline
\textbf{100}& $2$ & $-0.16$ & $\mathbf{-0.005}$ \\
\hline
\textbf{150}& $2$ & $-0.18$ & $\mathbf{-0.005}$ \\
\hline
\textbf{200}& $2$ & $-0.20$ & $\mathbf{-0.006}$ \\
\hline
\end{tabular}}\hspace{0.5in}
\caption{ Comparison with Stochastic Greedy(SG) algorithm. The $\epsilon$ parameter controls the accuracy of SG. A negative number means \gc is better than SG.}
\label{table:SG_MCP}
 \end{table}

\subsection{Results on Max Vertex Cover (MVC)}
\label{app:mvc}

\begin{figure}[t]
\centering
\subfloat[Gowalla]{
	\label{fig:gowalla_time_mvc}
	\includegraphics[width=1.33in]{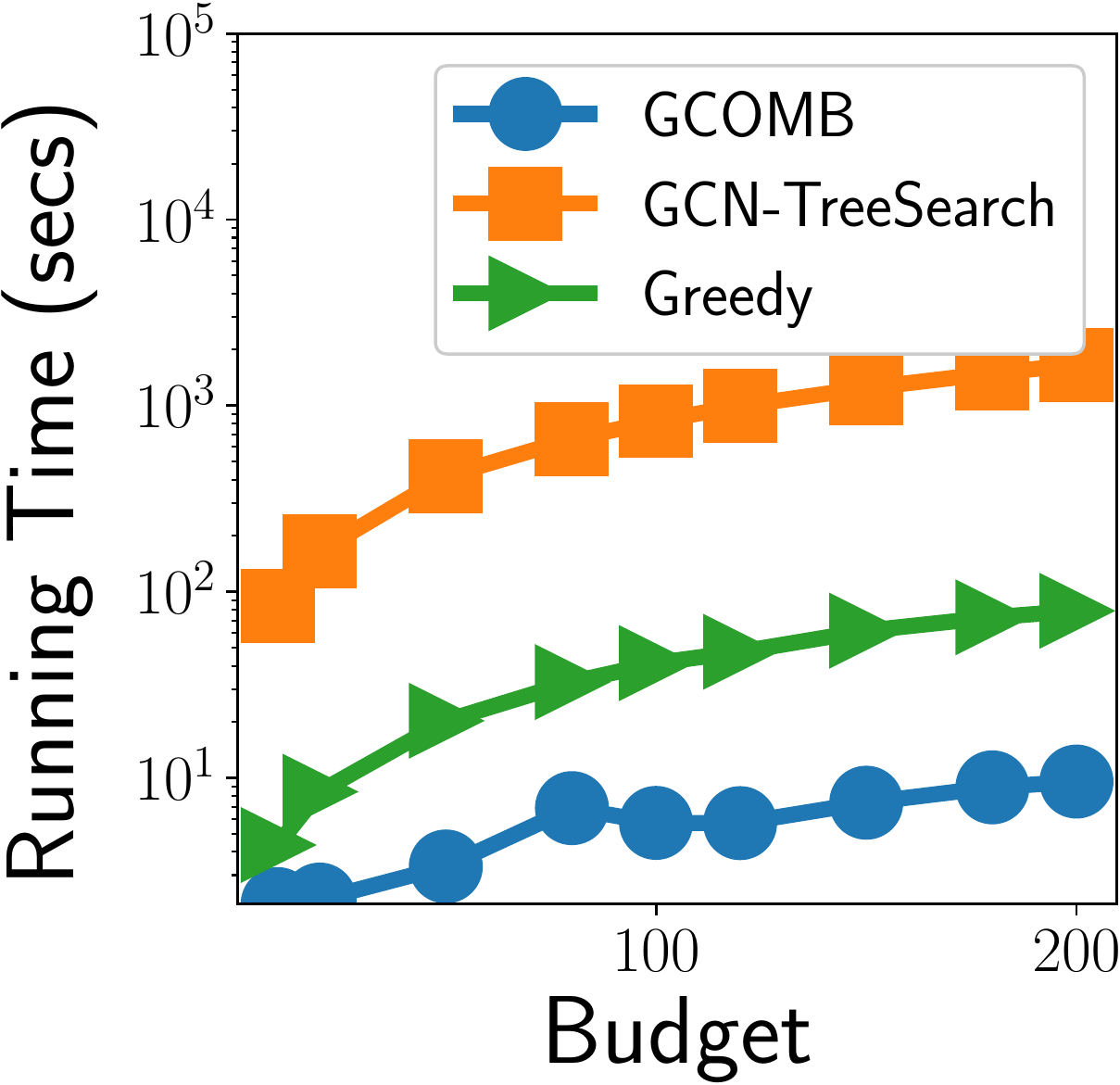}}
\subfloat[YouTube]{
	\label{fig:yt_time_mvc}
	\includegraphics[width=1.33in]{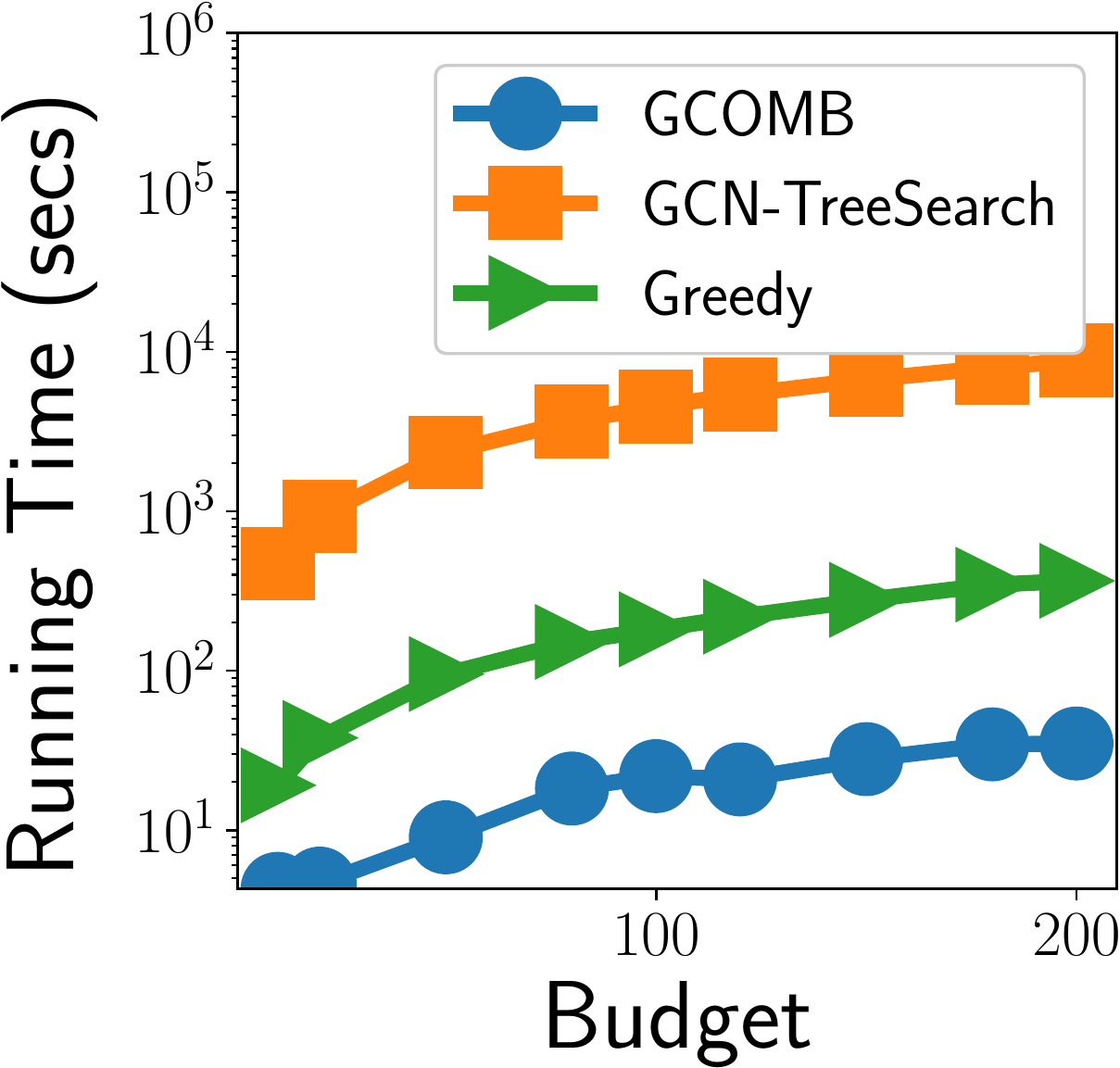}}
\subfloat[Twitter-ego]{
	\label{fig:twt-ego_mvc}
	\includegraphics[width=1.33in,height=1.28in]{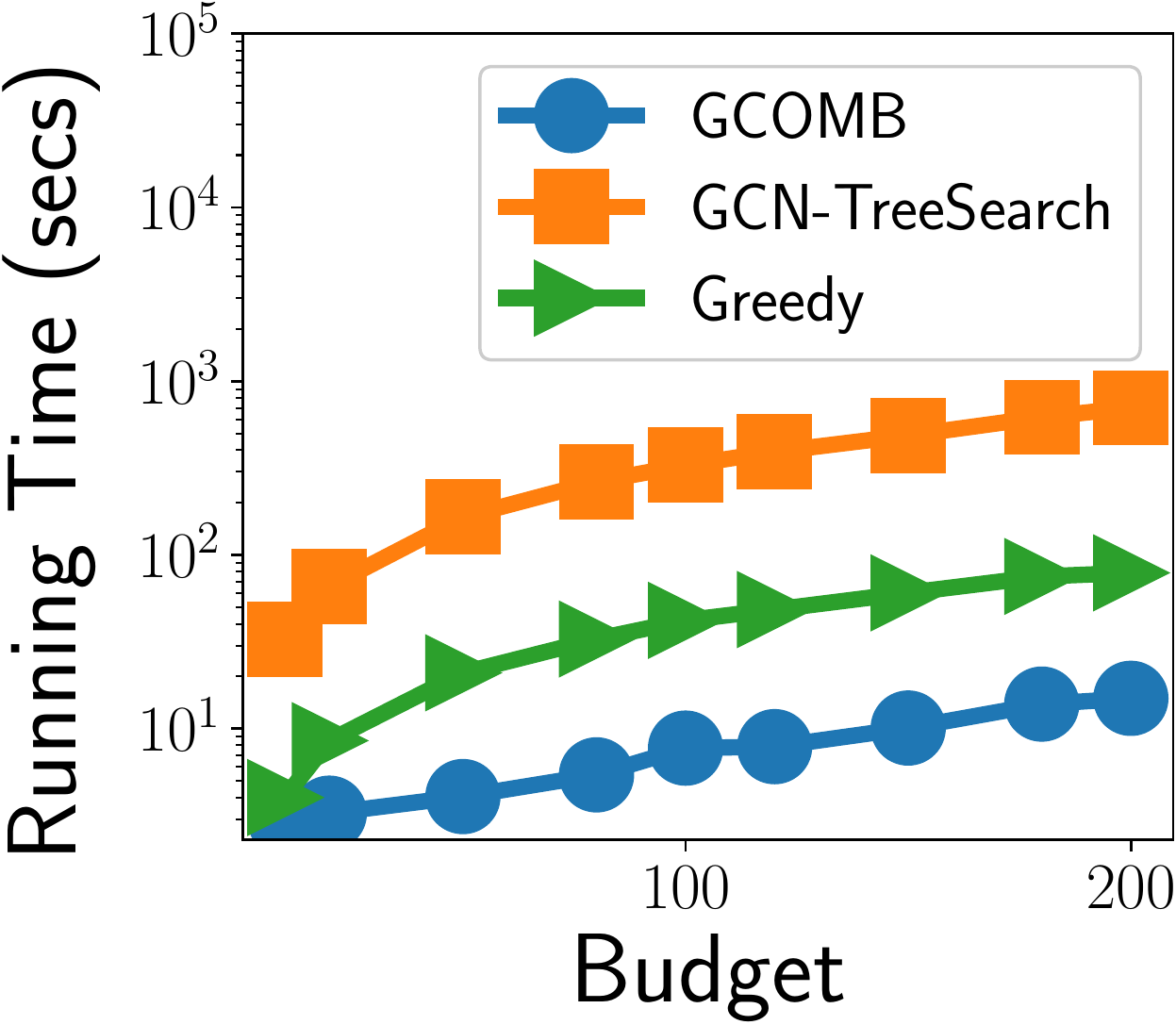}}
	
	\caption{MVC: Running times of \gc and \gt against the greedy approach in (a) Gowalla  (b) YouTube and c) Twitter-Ego.}
\end{figure}

\begin{table}[t]
\centering
\scalebox{0.8}{
\begin{tabular}{| c|c |c | c | c |c|}
\hline
\textbf{Graph}& \textbf{Greedy}& \textbf{\sv} & \textbf{\gt}  & \textbf{\gc} \\
\hline
\textbf{BA-10k} & $\textbf{0.11}$ & $0.096$ & $0.109$ & $\textbf{0.11}$\\
\hline
\textbf{BA-20k}& $\textbf{0.0781}$ & $0.0751$ & $\textbf{0.0781}$ & $\textbf{0.0781}$ \\
\hline
\textbf{BA-50k} & $\textbf{0.0491}$& $NA$& $0.0490$ & $\textbf{0.0491}$ \\
\hline
\textbf{BA-100k} & $\textbf{0.0346}$ & $NA$ & $0.0328$ & $\textbf{0.0346}$\\
\hline
\textbf{Gowalla} & $\textbf{0.081}$ & $NA$ & $\textbf{0.081}$ & $\textbf{0.081}$\\
\hline
\textbf{YouTube} & $\textbf{0.060}$& $NA$ & $\textbf{0.060}$ & $\textbf{0.060}$\\
\hline
\textbf{Twitter-ego} & $\textbf{0.031}$& $NA$ & $\textbf{0.031}$ & $\textbf{0.031}$\\
\hline
\end{tabular}}
\caption{Coverage achieved in the Max Vertex Cover (MVC) problem. The best result in each dataset is highlighted in bold. 
\label{table:mvc_gcomb_soa}}
\end{table}

To benchmark the performance in MVC, in addition to real datasets, we also use the Barabási–Albert (BA) graphs used in \sv \cite{dai2017learning}.

\textbf{Barabási–Albert (BA):} In BA, the default edge density is set to $4$, i.e., $|E|=4|V|$. We use the notation BA-$X$ to denote the size of the generated graph, where $X$ is the number of nodes.

For synthetic datasets, all three techniques are trained on BA graphs with $1k$ nodes. For real datasets, the model is trained on Brightkite. Table \ref{table:mvc_gcomb_soa} presents the coverage achieved at $b=30$. Both \gc and \gt produce results that are very close to each other and slightly better than \sv. As in the case of MCP, \sv ran out of memory on graphs larger than BA-20k. 

To analyze the efficiency, we next compare the prediction times of \gc with Greedy (Alg~\ref{alg:greedy}) and \gt. Figs.~\ref{fig:gowalla_time_mvc}-\ref{fig:twt-ego_mvc} present the prediction times against budget. Similar to the results in MCP, \gc is one order of magnitude faster than Greedy and up to two orders of magnitude faster than \gt. 

\begin{figure}[t]
\subfloat[TW and FS]{
	\label{fig:vb}
	\includegraphics[width=1.42in]{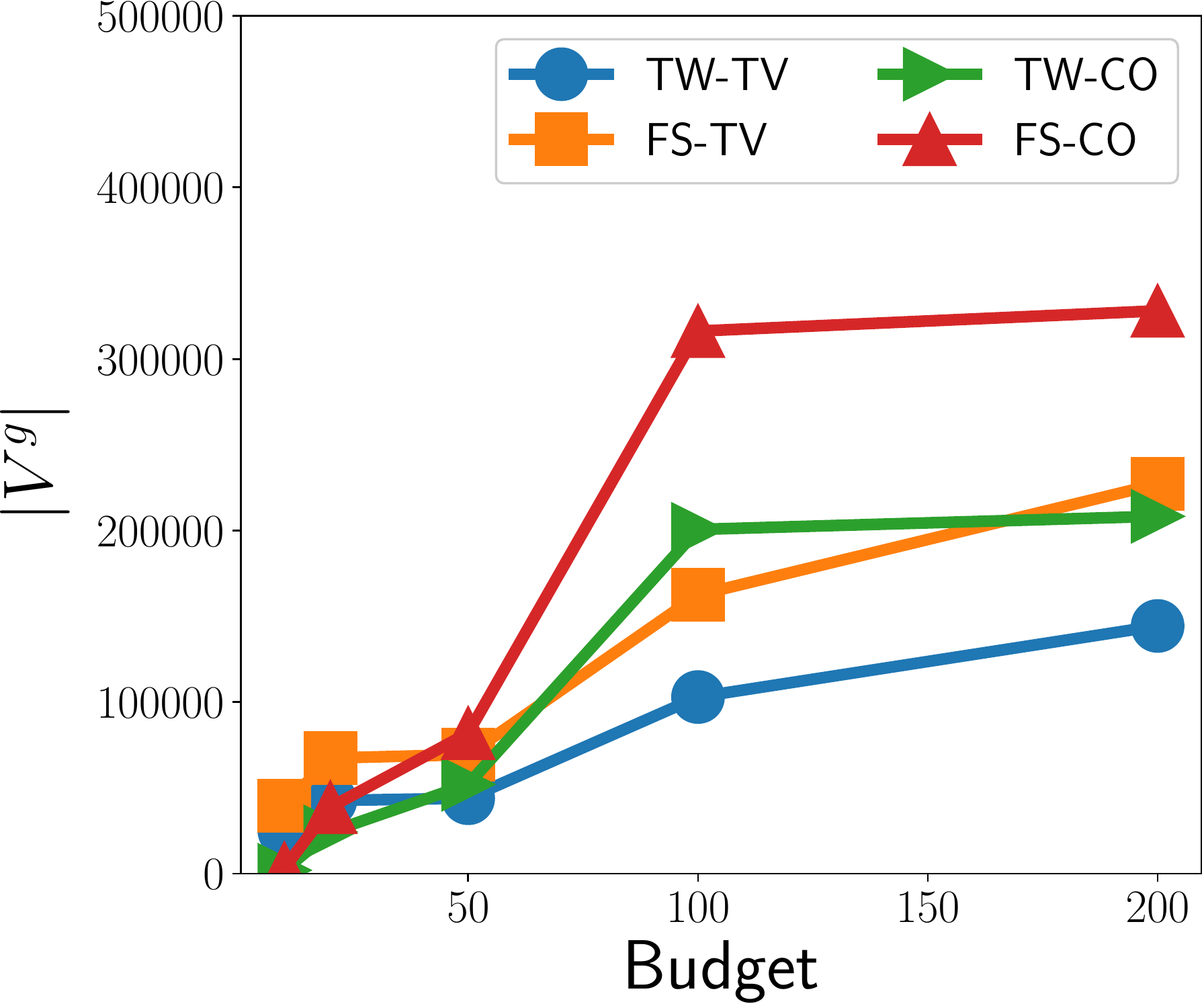}}
\subfloat[Quality Vs. Training]{
	\label{fig:trainingquality}
	\includegraphics[width=1.26in]{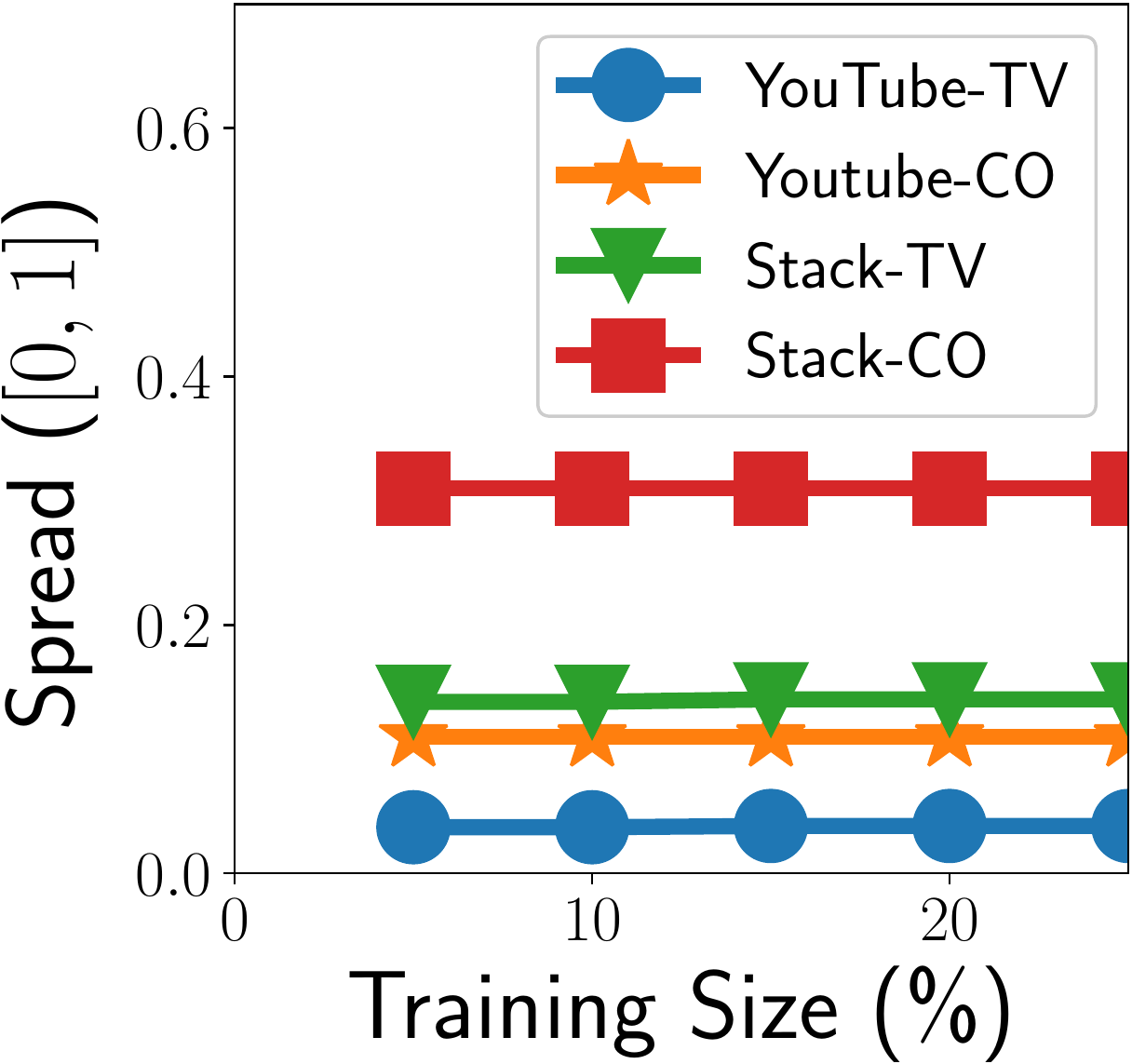}}
\subfloat[Gowalla]{
	\label{fig:sampling}
	\includegraphics[width=1.33in]{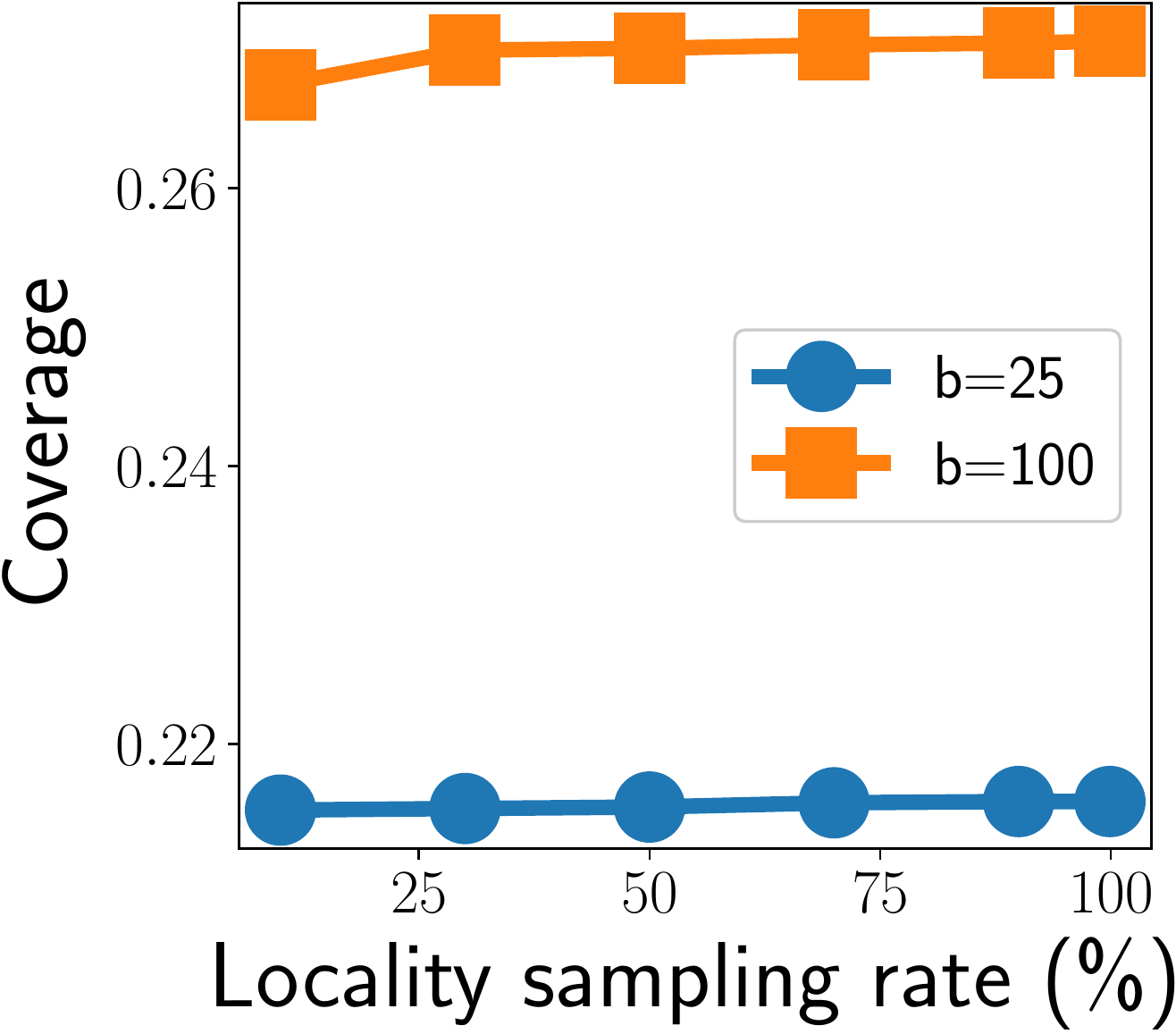}}
\subfloat[Gowalla]{
	\label{fig:dimension}
	\includegraphics[width=1.33in]{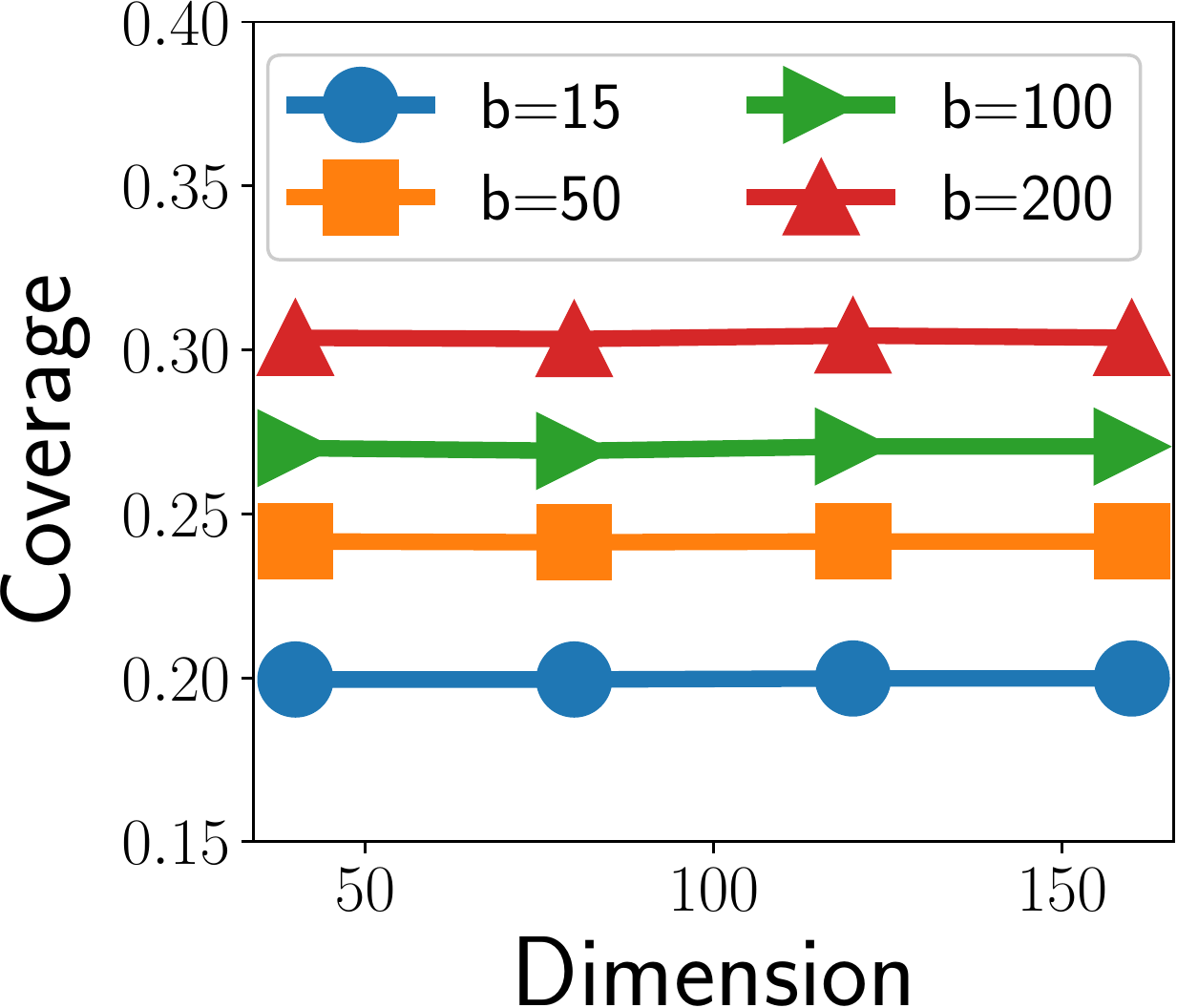}}
\caption{(a) Number of nodes included in $V^g$ in Twitter and Friendster for different budgets $b$. (b) Impact of training set size on spread quality in IM. (c-d) Effect of sampling rate and embedding dimension across different budgets on MCP coverage.}
\end{figure}

\vspace{-0.10in}
\subsection{Impact of Parameters}
\label{app:parameters_exp}
\subsubsection{Size of training data} In Fig.~\ref{fig:trainingquality}, we evaluate the impact of training data size on expected spread in IM. The budget for this experiment is set to 20. We observe that even when we use only $5\%$ of YT to train, the result is almost identical to training with a $25\%$ subgraph. This indicates \gc is able to learn a generalized policy even with small amount of training data.
\vspace{-0.05in}

\subsubsection{Effect of sampling rate }We examine how the sampling rate in locality computation affects the overall coverage in MCP. In Fig. \ref{fig:sampling}, with the increase of samples, the accuracy at $b=100$ increases slightly. At $b=25$, the increase is negligible. This indicates that our sampling scheme does not compromise on quality.
\subsubsection{Dimension} 
We vary the GCN embedding dimension from $40$ to $160$ and measure its impact on coverage in MCP (Fig.~\ref{fig:dimension}). We observe minute variations in quality, which indicates that \gc is robust and does not require heavy amount of parameter optimization.

\end{document}